\documentclass{article}

\usepackage{microtype}
\usepackage{graphicx}
\usepackage{subfigure}
\usepackage{booktabs} % for professional tables

\usepackage{hyperref}
\usepackage{url}

\usepackage[accepted]{icml2023}

\usepackage{amsmath}
\usepackage{amssymb}
\usepackage{mathtools}
\usepackage{amsthm}

\usepackage[capitalize,noabbrev]{cleveref}
\usepackage{breakcites}
\usepackage{thm-restate}

\usepackage{natbib}

\theoremstyle{plain}
\newtheorem{theorem}{Theorem}[section]

\newtheorem{lemma}[theorem]{Lemma}
\newtheorem{corollary}[theorem]{Corollary}
\theoremstyle{definition}
\newtheorem{definition}[theorem]{Definition}

\theoremstyle{remark}

\newtheorem{fact}[theorem]{Fact}
\newtheorem{problem}[definition]{Problem}

\newcommand{\Kf}{\kappa}
\newcommand{\R}{\mathbb{R}}

\newcommand{\ZZ}{\mathbb{Z}}

\newcommand{\K}{\mathcal{K}}

\newcommand{\norm}[1]{\left\|#1\right\|}

\DeclareMathOperator{\poly}{poly}
\DeclareMathOperator{\argmax}{argmax}

\newcommand{\epstwo}{\frac{\epsilon}{2}}

\newcommand{\sumppcondition}[1]{\sum\limits_{m\in M, \norm{x^* -m}_2^2 \le #1}}
\newcommand{\sumppconditionp}[1]{\sum\limits_{m\in M, \norm{x^* -m}_2^2 > #1}}

\newcommand{\N}{\mathcal{N}}

\usepackage[textsize=tiny]{todonotes}

\icmltitlerunning{Dimensionality Reduction for General KDE Mode Finding}

\begin{document}

\twocolumn[
\icmltitle{Dimensionality Reduction for General KDE Mode Finding}



\icmlsetsymbol{equal}{*}

\begin{icmlauthorlist}
\icmlauthor{Xinyu Luo}{equal,uni1}
\icmlauthor{Christopher Musco}{equal,uni2}
\icmlauthor{Cas Widdershoven}{equal,uni3}
\end{icmlauthorlist}

\icmlaffiliation{uni1}{Department of Computer Science, Purdue University, Indiana, USA}
\icmlaffiliation{uni2}{Tandon School of Engineering, New York University, New York, USA}
\icmlaffiliation{uni3}{State Key Laboratory of Computer Science, Institute of Software, Chinese Academy of Sciences, Beijing, China}

\icmlcorrespondingauthor{Xinyu Luo}{luo466@purdue.edu}
\icmlcorrespondingauthor{Christopher Musco}{cmusco@nyu.edu}
\icmlcorrespondingauthor{Cas Widdershoven}{cas.widdershoven@ios.ac.cn}

\icmlkeywords{Kernel density estimation, Johnson-Lindenstrauss, Sketching, Approximation algorithm, Mode finding, Kirszbraun theorem}

\vskip 0.3in
]



\printAffiliationsAndNotice{\icmlEqualContribution} % otherwise use the standard text.

\begin{abstract}
Finding the mode of a high dimensional probability distribution $\mathcal{D}$ is a fundamental algorithmic problem in statistics and data analysis. There has been particular interest in efficient methods for solving the problem when $\mathcal{D}$ is represented as a \emph{mixture model} or \emph{kernel density estimate}, although few algorithmic results with worst-case approximation and runtime guarantees are known. 
In this work, we significantly generalize a result of \cite{LeeLiMusco:2021} on mode approximation for Gaussian mixture models. We develop randomized dimensionality reduction methods for mixtures involving a broader class of kernels, including the popular logistic, sigmoid, and generalized Gaussian kernels. As in Lee et al.'s work, our dimensionality reduction results yield quasi-polynomial algorithms for mode finding with multiplicative accuracy $(1-\epsilon)$ for any $\epsilon > 0$. 
Moreover, when combined with gradient descent, they yield efficient practical heuristics for the problem.
%
% Moving beyond the Gaussian case requires leveraging work on generalized mean-shift methods for mode finding.
% as well as recent results on algorithms for computing Lipschitz extensions of high-dimensional functions.
In addition to our positive results, we prove a hardness result for box kernels, showing that there is no polynomial time algorithm for finding the mode of a kernel density estimate, unless $\mathit{P} = \mathit{NP}$. Obtaining similar hardness results for kernels used in practice (like Gaussian or logistic kernels) is an interesting future direction. 
\end{abstract}

\section{Introduction}
% \label{submission}

We consider the basic computational problem of finding the mode of a high dimensional probability distribution $\mathcal{D}$ over $\R^d$. Specifically, if $\mathcal{D}$ has probability density function (PDF) $p$, our goal is to find any $x^*\in \R^d$ such that:
\begin{align*}
    x^* \in \argmax_{x\in \R^d} p(x)
\end{align*}
A natural setting for this problem is when $\mathcal{D}$ is specified as a \emph{kernel density estimate} (KDE) or \emph{mixture distribution} \cite{Scott:2015, Silverman:2018}. In this setting, we are given a set of points $M\in \R^d$ and a non-negative kernel function $\kappa:\R^d\times \R^d \rightarrow \R^+$ and our PDF equals:
\begin{align*}
p(x) = \K_M(x) = \frac{1}{|M|}\sum_{m\in M} \Kf (x,m).
\end{align*}
As is typically the case, we will assume that $\Kf$ is \emph{shift-invariant} and thus only depends on the difference between $x$ and $m$, meaning that it can be reparameterized as $\Kf(x-m)$. A classic example of a shift-invariant KDE is any mixture of Gaussians distribution, for which $\Kf(x-m) = C\cdot e^{-\|x-m\|_2^2}$ is taken to be the Gaussian kernel. Here $C = \pi^{-d/2}$ is a normalizing constant.
Kernel density estimates are widely used to approximate other distributions  in a compact way
% in both the statistical and machine learning communities 
\cite{BotevGrotowskiKroese:2010, KimScott:2012}, and they have been applied to applications ranging from image annotation \cite{YavlinskySchofieldRuger:2005}, to medical data analysis \cite{SheikhpourSarramSheikhpour:2016}, to outlier detection \cite{KamalovLeung:2020}. The specific problem of finding the mode of a KDE has found applications in object tracking \cite{ShenBrooksHengel:2007}, super levelset computation \cite{PhillipsWangZheng:2015}, typical object finding \cite{GasserHallPresnell:1997}, and more \cite{LeeLiMusco:2021}.

\subsection{Prior Work}
Despite its many applications, the KDE mode finding problem presents a computational challenge in high-dimensions. For any practically relevant kernel $\Kf$ (e.g.,  Gaussian) there are no known algorithms with runtime polynomial in both $n$ and $d$ for KDEs on $n = |M|$ base points. This is even the case when we only want to find an \emph{$\epsilon$-approximate} mode for some $\epsilon \in (0,1)$, i.e. a point $\tilde{x}^*$ satisfying
\begin{align*}
    \K_M(\tilde{x}^*) \geq (1-\epsilon)\max_{x\in \R^d} \K_M(x),
\end{align*}
There has been extensive work on heuristic local search methods like the well-known ``mean-shift'' algorithm \cite{Carreira-Perpinan:2000,Carreira-Perpinan:2007}, which can be viewed as a variant of gradient descent, and often works well in practice. However, these methods do not come with theoretical guarantees and can fail on natural problem instances. 

While polynomial time methods are not known, for some kernels it is possible to provably solve the $\epsilon$-approximate mode finding problem in \emph{quasi-polynomial} time. For example, Shenmaier's work on \emph{universal approximate centers} for clustering can be used to reduce the problem to evaluating the quality of a quasi-polynomial number of candidate modes \cite{Shenmaier:2019}. For the Gaussian kernel, the total runtime is $d\cdot 2^{O(\log^2 n)}$ for constant $\epsilon$. Similar runtimes can be obtained by appealing to results on the approximate Carathéodory problem \cite{BlumHar-PeledRaichel:2019,Barman:2015}.
% \footnote{While we are not aware of direct prior work on this approach, it is straightforward. Approximate Carathéodory results show that any point in the convex hull of a point set can be approximately written as a linear combination of a subset of those points whose size only depends on the approximation factor. For typical kernels, the mode of a KDE must be contained in the convex hull of $M$, and thus can be written as a linear combination of a small number of points. We can obtain quasi-polynomial time via brute-force search over coefficients in the linear combination.}

More recently, Lee et al. explore \emph{dimensionality reduction} as an approach to obtaining quasi-polynomial time algorithms for KDE mode finding \cite{LeeLiMusco:2021}. Their work shows that, for the Gaussian kernel, any high-dimensional KDE instance can be reduced to a \emph{lower dimensional instance} using randomized dimensionality reduction methods -- specifically Johnson-Lindenstrauss  projection. An approximate mode for the lower dimensional problem can then be found with a method that depends exponentially on the dimension $d$, and finally, the low-dimensional solution can be ``mapped back'' to the high-dimensional space\footnote{Methods that run in time exponential in $d$ are straightforward to obtain via discretization/brute force search. See \Cref{sec:low-dim-prob}.}.
Ultimately, the result in \cite{LeeLiMusco:2021} allows all dependencies on $d$ to be replaced with terms that are polynomial in $\log(n)$ and $\epsilon$. The conclusion is that the mode of a Gaussian KDE can be approximated to accuracy $\epsilon$ in time $O\left(ndw + 2^{w}\right)$, where $w = \poly(\log n,1/\epsilon)$. The leading $ndw$ term is the cost of performing the dimensionality reduction.

In addition to nearly matching prior quasi-polynomial time methods in theory (e.g., Shenmaier's approach), there are a number of benefits to an approach based on dimensionality reduction. For one, sketching directly reduces the space complexity of the mode finding problem, and vectors sketched with JL random projections can be useful in other downstream data analysis tasks. Another benefit is that dimensionality reduction can speed up even heuristic algorithms: instead of using a brute-force approach to solve the low-dimensional KDE instance, a practical alternative is to apply a local search method, like mean-shift, in the low-dimensional space. 
This approach sacrifices theoretical guarantees, but can lead to faster practical algorithms. 

\subsection{Our Results}
% \paragraph{Dimensionality Reduction.}
The main contribution of our work is to generalize the dimensionality reduction results of \cite{LeeLiMusco:2021} to a much broader class of kernels, beyond the Gaussian kernel studied in that work. In particular, 
% in Section \ref{sec:dimred}
we introduce a carefully defined class of kernels called ``relative-distance smooth kernels''. This class includes the Gaussian kernel, as well as the sigmoid, logistic, and any generalized Gaussian kernel of the form $\kappa(x,y) = e^{-\|x-y\|_2^\alpha}$ for $\alpha > 0$. See Definition \ref{sec:dimred} for more details. Our first result (\Cref{lem:rds_main}) is that, for any relative-distance smooth kernel, we can approximate the \emph{value} of the mode $\max_x \K_M(x)$ up to multiplicative error $(1-\epsilon)$ by solving a lower dimensional instance obtained by sketching the points in $M$ using a Johnson-Lindenstrauss random projection. The required dimension of the projection is $O(\log^c(n)/\epsilon^2)$, where $c$ is a constant depending on parameters of the kernel $\Kf$. For most commonly used relative-distance smooth kernels, including the Gaussian, logistic, and sigmoid kernels, $c=3$. This leads to a dimensionality reduction that is completely independent of the original problem dimension $d$ and only depends polylogarithmically on the number of points in the KDE, $n$.

Moreover, in Section \ref{sec:modrec}, we show how to recover an approximate mode $\tilde{x}$ satisfying $K_M(\tilde{x}) \geq (1-\epsilon)\max_x \K_M(x)$ from the solution of the low-dimensional sketched problem. When the kernel satisfies an additional convexity property, recovery can be performed in $O(nd)$ time using a generalization of the mean-shift algorithm used in \cite{LeeLiMusco:2021}. When the kernel does not satisfy the property, we obtain a slightly slower method using a recent result of \cite{BiessKontorovichMakarychev:2019} on constructive Lipschitz extensions. One consequence of our general results is the following claim for a number of common kernels:

\begin{theorem}
\label{cor:common_recover}
Let $\K_M = (\Kf, M)$ be a be a KDE on $n=|M|$ points in $d$ dimensions, where $\Kf$ is a Gaussian, logistic, sigmoid,  Cauchy\footnote{For the Cauchy kernel, we can actually obtain a better bound with dimension $w = O\left(\frac{\log(n/\epsilon)}{\epsilon^2}\right)$. See \Cref{cor:cauchy_dim_reduc}.}, or generalized Gaussian kernel with parameter $\alpha \leq 1$. Let $\Pi$ be a random JL matrix with $w = O\left(\frac{\log^2(n/\epsilon)\log(n/\delta)}{\epsilon^2}\right)$ rows and let $\tilde{x}$ be any point such that ${\K}_{\Pi M} (\tilde{x}) \geq (1-\beta)\max_{x \in \R^w} {\K}_{\Pi M} (x)$. Given $\tilde{x}$ as input, \Cref{alg:convex} runs in $O(nd)$ time and returns, with probability $1-\delta$, a point $x' \in \R^d$ satisfying:
\begin{align*}
    {\K}_{M} (x') \geq (1-\epsilon - \beta)\max_{x \in \R^d} {\K}_{M} (x).
\end{align*}
\end{theorem}
Above, $\Pi M$ is the point set $\Pi M = \{\Pi m \text{ for } m\in M\}$ and ${\K}_{\Pi M}$ is the low-dimensional 
KDE defined by $\Pi M$ and $\Kf$. \Cref{cor:common_recover} implies that an approximate high-dimensional mode can be found by (approximately) solving a much lower dimensional problem. The result exactly matches that of \cite{LeeLiMusco:2021} in the Gaussian case.

When combined with a simple brute-force method for maximizing ${\K}_{\Pi M}$, \Cref{cor:common_recover} immediately yields a quasi-polynomial time algorithm for mode finding. Again we state a natural special case of this result, proven in \Cref{sec:low-dim-prob}. 

\begin{theorem}
\label{lemma:quais}
Let $\K_M = (\Kf, M)$ be a be a KDE on $n$ points in $d$ dimensions, where $\Kf$ is a Gaussian, logistic, sigmoid, Cauchy, or generalized Gaussian kernel with  $\alpha \leq 1$. There is an algorithm which finds a point $\tilde{x}$ satisfying:
\begin{align*}
\K_M(\tilde{x}) \geq (1-\epsilon) \max_x \K_M(x)
\end{align*}
in $2^{\tilde{O}(\log^3 n/\epsilon^2)} + O(nd\log^3(n)/\epsilon^2)$ time. Here $\tilde{O}(x)$ denotes $O(x\log^c x)$ for constant $c$.
\end{theorem}

Interestingly, as in \cite{LeeLiMusco:2021}, the above result falls just short of providing a polynomial time algorithm: doing so would require improving the $\log^3 n$ dependence in the exponent to $\log n$. It is possible to achieve polynomial time by make additional assumptions. For example, if we assume that $\mathcal{K}_M(x^*) \geq \rho$ for some constant $\rho$, then dependencies on $\log(n)$ can be replaced with $\log(1/\rho)$ using existing \emph{coreset methods} \cite{LeeLiMusco:2021,PhillipsTai:2018}. 
However, the question still remains as to whether the general KDE mode finding problem can be solved in polynomial time for any natural kernel. Our final contribution is to take a step towards answering this question in the negative by relating the mode finding problem to the $k$-clique problem, and showing an NP-hardness result for box kernels (defined in the next section). Formally, in \Cref{sec:hardness}, we prove:
\begin{lemma}\label{lem:hardbox}
The problem of computing a $\frac{1}{n}$-approximate mode of a box kernel KDE is NP-hard.
\end{lemma}
Unfortunately, our lower bound does not extend to commonly used kernels like the Gaussian, logistic, or sigmoid kernels. Proving lower bounds (or finding polynomial time algorithms) for these kernels is a compelling future goal.

\paragraph{Paper Structure.} \Cref{sec:prelims} contains notation and definitions. In \Cref{sec:dimred} we provide our main dimensionality results for approximating the objective value for the mode. Then, in \Cref{sec:modrec}, we show how to recover a high-dimensional mode from a low-dimensional one, providing different approaches for when the kernel is convex and not. \Cref{sec:low-dim-prob} outlines a brute force method for finding an approximate mode in low dimensions. 
In \Cref{sec:hardness} we show that the approximate mode finding problem is NP-hard for box kernels. 
Finally, we provide experimental results in \Cref{sec:experiments}, confirming that dimensionality reduction combined with a heuristic mode finding method 
yields a practical algorithm for a variety of kernels and data sets.
% for modeand we conclude in \Cref{sec:discussion}. Proofs can be found in \Cref{appendix}.

% We note that \cite{LeeLiMusco:2021} also considers the case when we can make additional assumptions on the true mode $x^*$. For example, if we assume that $\mathcal{K}_M(x^*) \geq \rho$ for some constant $\rho$, then dependencies on $\log(n)$ can be replaced with $\log(1/\rho)$ using existing \emph{coreset methods} \cite{PhillipsTai:2018}. This leads to polynomial time mode finding methods when there is a significant mode in the data set -- i.e., a location with PDF value at least a constant instead of inversely proportional on $n$. Similar coreset techniques could be combined with the work in this note when $\rho$ is assumed to be large, but we do not discuss details.

% All of the above work leads to quasi-polynomial time methods in the worst case, and the 

\section{Preliminaries}\label{sec:prelims}

\paragraph{Notation.}
For our purposes, a kernel density estimate (KDE) is defined by a set of $n$ points (a.k.a. centers) $M \subset \R^d$ and a non-negative, shift-invariant kernel function. All of the kernels discussed in this work are also \emph{radial symmetric}. This means that we can actually rewrite the kernel function $\Kf$ to be a scalar function of the squared Euclidean distance $\|x-m\|_2^2$.\footnote{We let $\|\cdot \|_2^2$ denotes the squared Euclidean norm: $\|a\|_2^2 = \sum_{i=1}^d a_i^2$, where $a_i$ is the $i^\text{th}$ entry in the length $d$ vector $a$.}. Our KDE then has the  form:
\begin{align*}
    \K_M(x) = \frac{1}{n}\sum_{m\in M} C\cdot \Kf(\|x-m\|_2^2).
\end{align*}
We further assume that $\Kf: \R\rightarrow \R$ is non-increasing, so satisfies $\Kf(t) \geq \Kf(t') \geq 0$ for all $t' \geq t$.
In the expression above, $C$ is a normalizing constant that only depends on $\kappa$. It is chosen to ensure that  $\int_{t\in \R^d} C\cdot \kappa(t) \,dt = 1$ and thus $\mathcal{K}_M$ is a probability density function. The above function $\K_M(x)$ is invariant to scaling $\Kf$, so to ease notation we further assume that $\Kf(0) = 1$.
Note that since $\Kf$ is non-increasing, we thus always have that $\max_t \Kf(t) = \Kf(0) = 1$. We write $\Kf'$ to denote the first-order derivative of $\Kf$ (whenever it exists).

% Here $\Kf: \R\rightarrow \R$ satisfies $\Kf(t) \geq \Kf(t') \geq 0$ for all $t' \geq t$.

% meaning that $\Kf(x-m) = f(\|x-m\|_2^2)$ for some function $f: \R\rightarrow \R$ that satisfies $f(t) \geq f(t')$ for all $t' \geq t$. Here $\|\cdot \|_2^2$ denotes the squared Euclidean norm: $\|a\|_2^2 = \sum_{i=1}^d a_i^2$, where $a_i$ is the $i^\text{th}$ entry in the length $d$ vector $a$.
% Since we focus on radial symmetric kernels, we overload the $\Kf$ notation, and assume from here on out that $\Kf: \R\rightarrow \R$ is a non-increasing function and our KDE has the form\footnote{In some prior work, the kernel parameterized by squared Euclidean distance has been called the \emph{profile} of the kernel \cite{Cheng:1995}. We simply take this form to be the canonical form of the kernel and refer to it as the \emph{kernel} or \emph{kernel function}.}:
% \begin{align*}
%     \K_M(x) = \frac{1}{n}\sum_{m\in M} C\cdot \Kf(\|x-m\|_2^2).
% \end{align*}
% Here $C$ is a normalizing constant that only depends on $\kappa$. It is chosen to ensure that  $\int_{t\in \R^d} C\cdot \kappa(t) \,dt = 1$ and thus $\mathcal{K}_M$ is a probability density. The above function $\K_M(x)$ is invariant to scaling $\Kf$, so to ease notation we further assume that $\Kf(0) = 1$.
% Note that since $\Kf$ is non-increasing, we thus always have that $\max_t \Kf(t) = \Kf(0) = 1$. We write $\Kf'$ to denote the first-order derivative of $\Kf$ (whenever it exists).

Many common kernels are radial symmetric and non-increasing, so fit the form described above \cite{Silverman:2018,Altman:1992,ClevelandDevlin:1988}. We list a few:
\begin{align*}
   &\text{Gaussian: } \Kf(t) = e^{-t} \\
   &\text{Logistic: } \Kf(t) = \frac{4}{e^{\sqrt{t}} + 2 + e^{-\sqrt{t}}}\\
    &\text{Sigmoid: } \Kf(t) = \frac{2}{e^{\sqrt{t}} + e^{-\sqrt{t}}} \\
    &\text{Cauchy: } \Kf(t) = \frac{1}{1+t} \\
    &\text{Generalized Gaussian: } \Kf(t) = e^{-t^\alpha} \\
    &\text{Box: } \Kf(t) = 1\text{ for $|t| \leq 1$},\, \Kf(t) = 0 \text{ otherwise}.\\
    % &\text{Triangle: } \Kf(t) = \max(0, 1 - \sqrt{t})\\
    &\text{Epanechnikov: } \Kf(t) = \max(0, 1 - t)
\end{align*}

We are interested in finding a value for $x$ which maximizes or approximately maximizes the kernel density estimate $\mathcal{K}_M(x)$. Again since the problem is invariant to positive scaling, we will consider the problem of maximizing the unnormalized KDE, which we denote by $\bar{\K}_M(x)$:
\begin{align*}
    \bar{\K}_M(x) = \sum_{m\in M} \Kf(\|x-m\|_2^2) = \frac{n}{C}\cdot {\K}_M(x) 
\end{align*}

Our general dimensionality reduction result depends on a parameter of $\Kf$ that we call the ``critical radius''. For common kernels we later show how to bound this parameter to obtain specific dimensionality reduction results. 
% that match \cite{LeeLiMusco:2021} for the Gaussian kernel and also apply to other kernels.
\begin{definition}[$\alpha$-critical radius, $\xi_\Kf(\alpha)$]
    For any non-increasing kernel function $\Kf: \R\rightarrow \R$, the $\alpha$-critical radius $\xi_\Kf(\alpha)$ is the smallest value of $t$ such that $\Kf(t) \leq \alpha$.
\end{definition}
Note that for any $t \geq \xi_\Kf(\alpha)$, we have that $\Kf(t) \leq \alpha$. 
The value of $\xi_\Kf(\epsilon/2n)$ and $\xi_\Kf(1/n)$ will be especially important in our proofs. Specifically, since  $\kappa$ is assumed to have $\kappa(0) = 1$, it is easy to check that any mode for $\K$ must lie within squared distance $\xi_\Kf(1/n)$ from at least one point in $M$, a region which we will call the \emph{critical area}. We will use this fact.
% Since $\kappa$ is assumed to have $\kappa(0) = 1$, we always have that $\max_x \bar{\K}_M(x) \geq 1$, as $\bar{\K}_M(x) \geq 1$ is already achieved by choosing $x$ to be any point in the set $M$. It follows that any mode $x^*$ of $\bar{\K}_M$ must be within squared distance $\xi_\Kf(1/n)$ from at least one point in $M$. If not, the point would clearly have $\bar{\K}_M(x^*) < 1$. Consider the union of balls of radius $\sqrt{\xi_\Kf(1/n)}$ around each point in $M$. We call this set the \emph{critical area} of the KDE, and can restrict our search for approximate modes to this critical area. 

\paragraph{Johnson-Lindenstrauss Lemma.}
% \label{subsec:jl}
Our results leverage the Johnson-Lindenstrauss (JL) lemma, which shows that a set of high dimensional points can be mapped into a space of much lower dimension in such a way that distances between the points are nearly preserved. We use the standard variant of the lemma where the mapping is an easy to compute random linear transformation \cite{Achlioptas:2001,DasguptaGupta:2003}. Specifically, we are interested in random transformations satisfying the following guarantee:

\begin{definition}[$(\gamma, n, \delta)$-Johnson-Lindenstrauss Guarantee]
\label{def:jl}
A randomly selected matrix $\Pi \in \R^{w\times d}$ satisfies the $(\gamma, n, \delta)$-JL guarantee for positive error parameter $\gamma$, if for any $n$ data points $v_1,..., v_n\in \R^d$, with probability $1 - \delta$,
    \begin{align} \label{JL}
        \norm{v_i - v_j}_2^2 \le \norm{\Pi v_i - \Pi v_j}_2^2 \le (1 + \gamma)\norm{v_i - v_j}_2^2
    \end{align}
    for all pairs $i, j \in \{1,..., n\}$ simultaneously.
\end{definition}
Note that we require one-sided error: most statements of the JL guarantee have a $(1-\gamma)$ factor on the left side of the inequality. This is easily removed by scaling $\Pi$ by $\frac{1}{1-\gamma}$. It is well known that \Cref{def:jl} is satisfied by a properly i.i.d. random Gaussian or random $\pm 1$ matrix with
\begin{align*}
    w = O\left(\frac{\log (n/\delta)}{\min(1,\gamma^2)}\right)
\end{align*}
rows, and this is tight \cite{LarsenNelson:2017}. General sub-Gaussian random matrices also work, as well as constructions that admit faster computation of $\Pi v_i$ \cite{KaneNelson:2014,AilonChazelle:2009}.

\paragraph{Kirszbraun Extension Theorem.}
We also rely on a classic result of \cite{Kirszbraun:1934}. Let $H_1$ and $H_2$ be Hilbert spaces. Kirszbraun's theorem states that if $S$ is a subset of $H_1$, and $f: S \rightarrow H_2$ is a Lipschitz-continuous map, then there is a Lipschitz-continuous map ${g}: H_1 \rightarrow H_2$ that extends\footnote{I.e. $g(s) = {f}(s)$ for all $s\in S$.} $f$ and has the same Lipschitz constant. Formally, when applied to Euclidean spaces $\R^w$ and $\R^d$ we have:
\begin{fact}\label{theo:Kirs}
    (Kirszbraun  Extension Theorem). For any $\mathcal{S} \subset \R^w$, let $f: S \rightarrow \R^d$ be an L-Lipschitz function. That is $\forall x, y \in \mathcal{S}$, $\norm{f(x) - f(y)}_2 \le L\norm{x - y}_2$. Then, there always exists some function $g: \R^w \rightarrow \R^d$ such that:
    \begin{enumerate}
        \item $g(x) = f(x)$ for all $x \in \mathcal{S}$,
        \item $g$ is also L-Lipschitz. That is for all $x, y \in \R^w$, $\norm{{g}(x) - {g}(y)}_2 \le L\norm{x - y}_2$.
    \end{enumerate}
\end{fact}

\section{Dimensionality Reduction for Approximating the Mode Value}
\label{sec:dimred}
In this section, we show that using a JL random projection, we can reduce the problem of approximating the \emph{value} of the mode of a KDE in $d$ dimensions -- i.e., $\max_x \bar{\mathcal{K}}_M(x)$ -- to the problem of approximating the value of the mode for a KDE in $d'$ dimensions, where $d'$ depends only on $n$, $\Kf$, and the desired approximation quality. This problem of recovering the mode value is a prerequisite for the harder problem of recovering the \emph{location} of an approximate mode (i.e., a point $x^* \in \R^d$ such that $\K(x^*) \geq (1-\epsilon)\max_{x \in \R^d} \K(x)$), which is addressed in~\Cref{sec:modrec}.

We begin with an analysis for JL projections that bounds $d'$ based on generic properties of $\Kf$. Then, in~\Cref{subsec:expsmoothedreduc} we analyze these properties for specific kernels of interest, and prove that $d'$ is in fact small for these kernels -- specifically, it depends just polylogarithmically on $n$ and polynomially on the approximation factor $\epsilon$.
% , matching known results for the Gaussian kernel \cite{LeeLiMusco:2021}.
Our general result follows:

\begin{restatable}{theorem}{mapdownapproxmode}\label{thm:mapdownapproxmode}
Let $\K_M = (\Kf, M)$ be a $d$-dimensional KDE on a differentiable kernel as defined in~\Cref{sec:prelims} and let $0< \epsilon \leq 1$ be an approximation factor. Let $\xi \geq \xi_{\Kf}(\frac{\epsilon}{2n})$ and let $\Kf'_{\min} \leq \min_{0 \leq t \leq 2\xi} \frac{\Kf'(t) t}{\Kf(t)}$. Note that $\Kf'_{\min} \leq 0$ since $\Kf$ is assumed to be non-increasing. We can assume that $\Kf'_{\min} \neq 0$. Let $\gamma = -\frac{\epsilon}{2\Kf'_{\min}} > 0$. Then with probability $(1-\delta)$, for any  $\Pi \in \R^{w \times d}$ satisfying the $(\gamma, n+1, \delta)$-JL guarantee, we have:
\begin{equation}
\label{eq:theorem2_equation}
        (1-\epsilon) \max_{x \in \R^d} {\K}_M(x) \le \max_{x \in \R^w} {\K}_{\Pi M} (x) \le \max_{x \in \R^d} {\K}_{M}(x).
\end{equation}
Recall that a random $\Pi$ with $w = O\left(\frac{\log((n+1)/\delta)}{\min(1,\gamma^2)}\right)$ rows will satisfy the required $(\gamma, n+1, \delta)$-JL guarantee.
\end{restatable}
Note that in the theorem statement above, $\Pi M = \{\Pi m : m\in M\}$ denotes the point set $M$ with dimension reduced by multiplying each point in the set by $\Pi$. Our proof of \Cref{thm:mapdownapproxmode} is included in \Cref{appendix}. It leverages Kirszbraun's Exention theorem, and follows along the same lines in \cite{LeeLiMusco:2021}. However, we need to more carefully track the effect of properties of the kernel function $\Kf$, since we do not assume that it has the simple form of a Gaussian kernel.

With \Cref{thm:mapdownapproxmode} in place, we can apply it to any non-increasing differentiable kernel to obtain a dimensionality reduction result: we just need to compute a lower bound $\Kf'_{\min} \leq \min_{0 \leq t \leq 2\xi} \frac{\Kf'(t) t}{\Kf(t)}$. For some kernels we can do so directly.
For example, consider the Cauchy kernel, $\Kf(t) = \frac{1}{1+t}$. It can be shown that we can pick $\Kf'_{\min} = -1$ (since $\Kf'(t)t/\Kf(t) \geq -1$ for all $t$). Plugging into \Cref{thm:mapdownapproxmode} we obtain:

\begin{corollary}
\label{cor:cauchy_dim_reduc}
Let $\K_m = (\Kf, M)$ be a KDE and, for any $\delta,\epsilon \in (0,1)$, let $\Pi$ be a random JL matrix with $w = O\left(\frac{\log(n/\delta)}{\epsilon^2}\right)$ rows. If $\Kf$ is a Cauchy kernel, then with probability $1-\delta$, 
    \begin{align*}
    (1-\epsilon) \max_{x \in \R^d} {\K}_M(x) \le \max_{x \in \R^w} {\K}_{\Pi M} (x) \le \max_{x \in \R^d}  {\K}_{M}(x).
    \end{align*}
\end{corollary}

In the following subsection we will describe a broader class of kernels for which we can also obtain good dimensionality reduction results, but for which bounding $\Kf'_{\min}$ is a bit more challenging.

\subsection{Relative-Distance Smooth Kernels}\label{subsec:expsmoothedreduc}
Specifically, we consider a broad class of kernels, that included the Gaussian kernel:
\begin{definition}[Relative-distance smooth kernel]
\label{def:rds}
    A non-increasing differentiable kernel $\Kf$ is \emph{relative-distance smooth} if  there exist constants $c_1,d_1,q_1,c_2,d_2 > 0$ such that
    \begin{align*}
        c_1t^{d_1} - q_1 &\leq \frac{-\Kf'(t)t}{\Kf(t)} \leq c_2t^{d_2}& &\text{ for all }& t &\geq 0.
    \end{align*}
\end{definition}
In addition to the Gaussian kernel, this class includes other kernels commonly used in practice, like the logistic, sigmoid, and generalized Gaussian kernels:
% \begin{center}
% \begin{tabular}{l|l} 
% Gaussian: $\Kf(t) = e^{-t}$ & \parbox{0.34\linewidth}{$\begin{aligned}t^1 \leq \frac{-\Kf'(t)t}{\Kf(t)} = t \leq t^1\end{aligned}$}\\ 
% Logistic: $\Kf(t) = \frac{4}{e^{\sqrt{t}} + 2 + e^{-\sqrt{t}}}$ & \parbox{0.34\linewidth}{$\begin{aligned}\frac{1}{2} {t}^{1/2} - \frac{1}{2} \leq \frac{-\Kf'(t)t}{\Kf(t)} = \frac{(e^{\sqrt{t}} - 1)\sqrt{t}}{2(e^{\sqrt{t}} + 1)} \leq \frac{1}{2} {t}^{1/2}\end{aligned}$}\\ 
% Sigmoid: $\Kf(t) = \frac{2}{e^{\sqrt{t}} + e^{-\sqrt{t}}}$ & \parbox{0.34\linewidth}{$\begin{aligned}\frac{1}{2} {t}^{1/2}- \frac{1}{2} \leq \frac{-\Kf'(t)t}{\Kf(t)} = \frac{(e^{2\sqrt{t}} - 1)\sqrt{t}}{2(e^{2\sqrt{t}} + 1)} \leq \frac{1}{2} {t}^{1/2}\end{aligned}$}\\
% Generalized Gaussian  $\Kf(t) = e^{-t^\alpha}$ & \parbox{0.34\linewidth}{$\begin{aligned}\alpha t^\alpha \leq \frac{-\Kf'(t)t}{\Kf(t)} = \alpha t^{\alpha} \leq \alpha t^\alpha\end{aligned}$}\\ 
% \end{tabular}
% \end{center}
\begin{align*}
&\text{Gaussian: } t^1 \leq \frac{-\Kf'(t)t}{\Kf(t)} = t \leq t^1\\
&\text{Logistic: } \frac{{t}^{1/2}}{2}  - \frac{1}{2} \leq \frac{-\Kf'(t)t}{\Kf(t)} = \frac{(e^{\sqrt{t}} - 1)\sqrt{t}}{2(e^{\sqrt{t}} + 1)} \leq \frac{{t}^{1/2}}{2} \\
&\text{Sigmoid: } \frac{{t}^{1/2}}{2} - \frac{1}{2} \leq \frac{-\Kf'(t)t}{\Kf(t)} = \frac{(e^{2\sqrt{t}} - 1)\sqrt{t}}{2(e^{2\sqrt{t}} + 1)} \leq \frac{{t}^{1/2}}{2}\\
&\text{Generalized Gaussian: } \alpha t^\alpha \leq \frac{-\Kf'(t)t}{\Kf(t)} = \alpha t^{\alpha} \leq \alpha t^\alpha
\end{align*}

A few common non-increasing kernels, including the rational quadratic kernel, are \emph{not} relative distance smooth.
Our main result is that for \emph{any} {relative-distance smooth kernel}, we can sketch the KDE to dimension $w$ which depends only polylogarithically on $n = |M|$ and quadratically on $1/\epsilon$:
\begin{restatable}{lemma}{rdsmain}
\label{lem:rds_main}
    Let $\K_m$ be a KDE for a relative-distance smooth kernel $\Kf$ with parameters $c_1,d_1,q_1,c_2,d_2$. There is a fixed constant $c'$ such that if $\gamma = \frac{\epsilon}{c'}\log^{-d_2/d_1}\left(\frac{2n}{\epsilon}\right)$, then with probability $(1-\delta)$, for any  $\Pi \in \R^{w \times d}$ satisfying the $(\gamma, n+1, \delta)$-JL guarantee, \Cref{eq:theorem2_equation} holds. To obtain this JL guarantee, it suffices to take $\Pi$ to be a random JL matrix with $w = O\left(\frac{\log^{2d_2/d_1}(n/\epsilon)\log(n/\delta)}{\epsilon^2}\right)$ rows.
\end{restatable}

\Cref{lem:rds_main} is proven in \Cref{appendix}. It uses an intermediate result that bounds the $\frac{\epsilon}{2n}$-critical radius for any  relative-distance smooth kernel, which is required to invoke \Cref{thm:mapdownapproxmode}. Interestingly, the polylogarithmic factor in \Cref{lem:rds_main} only depends on the ratio of the parameters $d_2$ and $d_1$ of the relative-distance smooth kernel $\Kf$. For all of the example kernels discussed above, this ratio equals $1$, so we obtain a dimensionality reduction result exactly matching \cite{LeeLiMusco:2021} for the Gaussian kernel:
\begin{corollary}
\label{cor:common_dim_reduc}
Let $\K_m = (\Kf, M)$ be a KDE and, for any $\delta,\epsilon \in (0,1)$,  
    let $\Pi$ be a random JL matrix with $w = O\left(\frac{\log^2(n/\epsilon)\log(n/\delta)}{\epsilon^2}\right)$ rows. If $\Kf$ is a Gaussian, logistic, sigmoid kernel, or generalized Gaussian kernel, then with probability $1-\delta$, 
    \begin{align*}
    (1-\epsilon) \max_{x \in \R^d} {\K}_M(x) \le \max_{x \in \R^w} {\K}_{\Pi M} (x) \le \max_{x \in \R^d}  {\K}_{M}(x).
    \end{align*}
\end{corollary}

\section{Recovering an Approximate Mode in High Dimensions}\label{sec:modrec}

% \subsection{Mode Recovery}
In \Cref{sec:dimred}, we discussed how to convert a high dimensional KDE into a lower dimensional KDE whose mode has an approximately equal value. However, in applications, we are typically interested in computing a point in the high-dimensional space whose value is approximately equal to the value of the mode. I.e., using our dimensionality reduced dataset, we want to find some $\tilde{x}$ such that:
\begin{align*}
    \mathcal{K}_M(\tilde{x}) \geq (1-\epsilon) \max_x \mathcal{K}_M({x}).  
\end{align*}
We present two approaches for doing so. The first is based on Kirszbraun's extension theorem and the widely used mean-shift heuristic. It extends the approach of \cite{LeeLiMusco:2021} to a wider class of kernels -- specifically to any \emph{convex and non-increasing} kernel $\Kf$. This class contains most of the relative-distance smooth kernels discussed in~\Cref{subsec:expsmoothedreduc}, including the Gaussian, sigmoid, and logistic kernels, and generalized Gaussian kernels when $\alpha \leq 1$. It also includes common kernels like the Cauchy kernel, for which we have shown a strong dimensionality reduction results, and the Epanechnikov, biweight, and triweight kernels. Recall that we define $\Kf(t)$ so that $t$ represents the \emph{squared} Euclidean distance between two points; we specifically need $\Kf$ as defined in this way to be convex.

For non-convex kernels, we briefly discuss a second approach in \Cref{sec:nonconvexrecovery} based on recent work on explicit one point extensions of Lipschitz functions \cite{BiessKontorovichMakarychev:2019}. While less computationally efficient, this approach works for any non-increasing $\Kf$. Common examples of non-convex kernels include the tricube kernel $\Kf(t) = (1-t^{3/2})^3$ \cite{Altman:1992}, $\Kf(t) = 1-t^2$ \cite{comaniciu_meer_2000}, or any generalized Gaussian kernel with $\alpha > 1$. 

\subsection{Mean-shift for Convex Kernels}
\label{sec:convexrecovery}

\begin{algorithm}[tb]
\caption{Mean-Shift Algorithm}\label{alg:cap}

\begin{algorithmic}[1]
\REQUIRE Set of $n$ points $M \subset \R^d$, number of iterations $\tau$, differentiable kernel function $\Kf$.
\STATE Select initial point $x^{(0)} \in \R^d$
\STATE For $i = 0, ..., \tau - 1$:
\vspace{-1.25em}
\begin{align*}
x^{(i + 1)} = \sum_{m\in M} m\cdot \frac{\Kf'\left(\norm{x^{(i)} - m}_2^2\right)}{\sum_{j\in M}\Kf'\left(\norm{x^{(i)} - j}_2^2\right)}
\end{align*}
\vspace{-1em}
\STATE return $x^{(\tau)}$
\end{algorithmic}
\end{algorithm}

% \textcolor{blue}{Note: This paragraph are modified and added some contends about mean-shift for clustering. - Xinyu}
% We begin with our first approach for recovering a high-dimensional approximate mode by using the mean-shift algorithm. 
Based on ideas proposed by Fukunaga and Hostetler \cite{FukunagaHostetler:1975}, the mean-shift method is a commonly used heuristic for finding an approximate mode \cite{Cheng:1995}. The idea behind the algorithm is to iteratively refine a guess for the mode. At each update, a new guess $x^{(i+1)}$, is obtained by computing a weighted average of all points in $M$ that define the KDE. Points that are closer to the previous guess $x^{(i)}$ are included with higher weight than points that are further. The exact choice of weights depends on the first derivative $\Kf'(t)$, where $t$ is the distance from the current mode to a point in $M$. For any non-increasing, convex kernel, $\Kf'(t)$  is non-positive and decreasing in magnitude -- i.e.,  $|\Kf'(t)|$ is largest for $t$ close to $0$, which ensures that points closest to the current guess for the mode are weighted highest when computing the new guess\footnote{Note that for the Gaussian kernel, $\Kf(t) = e^{-t}$, so $|\Kf'(t)| = \Kf(t)$. So the method presented here is equivalent to the version of mean-shift used in prior work on dimensionality reduction for mode finding \cite{LeeLiMusco:2021}}. We include pseudocode for mean-shift as \Cref{alg:cap}. The method can be alternatively viewed as an instantiation of gradient ascent for the KDE mode objective with a specifically chosen step size -- we do not discuss details here.

A powerful property of the mean-shift algorithm is that it always converges for kernels that are non-increasing and convex. In fact, it is known to provide a monotonically improving solution. Specifically:

\begin{fact}[\citet{ComaniciuMeer:2002}]
	\label{theo:Meer}
	Let $x^{(0)}\in \R^d$ be an arbitrary starting point and let $x^{(1)}, \ldots, x^{(\tau)}$ be the resulting iterates of \Cref{alg:cap} run on point set $M$ with kernel $\Kf$. 
	If $\Kf$ is convex and non-increasing, then for any $i\in 1, \ldots, \tau$:
	\begin{align*}
		\K_M(x^{(i)}) \ge \K_M(x^{(i - 1)}).
	\end{align*} 
\end{fact}

In \Cref{appendix}, we use this fact to prove that with a (modified) mean-shift method, run for only a single iteration, we can translate any approximate solution for a dimensionality reduced KDE problem to a solution for the original high dimensional problem. Formally, we prove the following:

\begin{restatable}{theorem}{kdemeanshift}\label{thm:kde_meanshift}
    Let $M$ be a set of points in $\R^d$ and let $\K_M = (\Kf, M)$ be a KDE defined by a shift-invariant, non-increasing, and convex kernel function $\Kf$. Let $x^* \in \argmax_x \K_M(x)$. Let $\Pi\in \R^{w\times d}$ be a JL matrix as in \Cref{def:jl} and assume that $w$ is chosen large enough so that for all $a,b$ in the set $\{x^*\}\cup M$, 
    \begin{align}
    	\label{eq:lip_pi}
    	\|a - b\|_2^2 \leq \|\Pi a - \Pi b\|_2^2 
%    	\leq (1+\gamma)\|a - b\|_2^2
    \end{align}
    \begin{align*}
        \text{and } (1-\epsilon) \max_{x \in \R^d} {\K}_M(x) \le \max_{x \in \R^w} {\K}_{\Pi M} (x) \le \max_{x \in \R^d} {\K}_{M}(x).
    \end{align*}
    Let $\tilde{x} \in \R^w$ be an approximate maximizer for ${\K}_{\Pi M}$ satisfying ${\K}_{\Pi M} (\tilde{x}) \geq (1-\alpha)\max_{x \in \R^w} {\K}_{\Pi M} (x)$. Then if we choose $x' = \sum_{m \in M} m\cdot \frac{\Kf'(\norm{\tilde{x} - \Pi m}^2)}{\sum_{m \in M} \Kf'(\norm{\tilde{x} - \Pi m}^2)}$, we have:
    \begin{align*}
        {\K}_{M} (x') \geq (1-\epsilon - \alpha)\max_{x \in \R^d} {\K}_{M} (x).
    \end{align*}
\end{restatable}
Note that $x'$ above is set using a single-iteration of what looks like mean-shift. However, instead of using weights based on the distances of points in $M$ to a previous guess for a high-dimensional mode, we use distances between the points $\Pi M$ in our low-dimensional space to the approximate low-dimensional mode, $\tilde{x}$. 
Also note that \Cref{thm:kde_meanshift} is independent of exactly how $\tilde{x}$ is computed -- it could be computed using brute force search, using an approximation algorithm tailored to low-dimensional problems, as in \cite{LeeLiMusco:2021}, or using a heuristic like mean-shift itself. 
% We will discuss this point further in \Cref{sec:low-dim-prob}, where we show that simple brute-force search methods combined with dimensionality reduction yield quasi-polynomial time algorithms for finding the mode for relative-distance smooth kernels. 

\begin{algorithm}
\caption{Mode Recovery for Convex Kernels}\label{alg:convex}

\begin{algorithmic}[1]
\REQUIRE Shift-invariant, non-increasing, and convex kernel function $\Kf$ with derivative $\Kf'$. Set of $n$ points $M \subset \R^d$, dimensionality reduction parameter $\gamma$, accuracy parameter $\alpha$, failure probability $\delta$.
\STATE Construct a random JL matrix $\Pi$ with $w = O\left(\frac{\log((n+1)/\delta)}{\min(1,\gamma^2)}\right)$ rows.
\STATE Construct a set of $n$ points $\Pi M \subset \R^w$ that contains $\Pi m$ for each $m\in M$.
\STATE Compute $\tilde{x}$ such that ${\K}_{\Pi M} (\tilde{x}) \geq (1-\alpha)\max_{x \in \R^w} {\K}_{\Pi M} (x)$.
\STATE Return $x' = \sum_{m \in M} m\cdot \frac{\Kf'(\norm{\tilde{x} - \Pi m}^2)}{\sum_{m \in M} \Kf'(\norm{\tilde{x} - \Pi m}^2)}$
% \State return $x'$
\end{algorithmic}
\end{algorithm}

For convex kernels, \Cref{thm:kde_meanshift} implies a strengthening of \Cref{thm:mapdownapproxmode} that allows for recovering an approximate mode, not just the value of the mode. Formally, the combined dimensionality reduction and recover procedure we propose is included as \Cref{alg:convex} and we have the following result on the its accuracy: 

\begin{restatable}{corollary}{fullconvex}\label{thm:fullconvex}
Let $\K_M = (\Kf, M)$ be a $d$-dimensional shift-invariant KDE as defined in~\Cref{sec:prelims} and let $\epsilon$ and $\gamma$ (which depends on $\Kf$ and $\epsilon$) be as in \Cref{thm:mapdownapproxmode}. If $\Kf$ is differentiable, non-increasing, and convex, then \Cref{alg:convex} run with parameters $\gamma$ and $\alpha$ returns $x'$ satisfying:
\begin{align*}
    {\K}_{M} (x') \geq (1-\epsilon - \alpha)\max_{x \in \R^d} {\K}_{M} (x).
\end{align*}
\end{restatable}

Note that Line 4 in \Cref{alg:convex} can be evaluated in $O(nw + nd)$ time. So our headline result, \Cref{cor:common_recover}, follows as a direct corollary.

\section{Solving the Low-Dimensional Problem}
\label{sec:low-dim-prob}
We next discuss a simple brute-force search method for approximate mode finding for any KDE with a continuous kernel function $\Kf$. The method has an exponential runtime dependency on the dimension, so its use for high-dimensional problems is limited, but combined with the dimensionality reduction techniques from~\Cref{sec:dimred} and the mode recovery techniques from~\Cref{sec:modrec}, it yields a quasi-polynomial mode finding algorithm for a large class of kernels.

Recall that the mode of a KDE $\K = (\Kf, M)$ with $|M|=n$ must lie within its critical area, i.e. in a ball of squared radius $\xi_{\Kf}(1/n)$ around one of the points in $M$ (where $\xi_{\Kf}(1/n)$ denotes the $1/n$-critical radius). For any $\delta > 0$ we define a finite \emph{$\delta$-covering} $\N(\K,\delta)$ to be a finite set of points such that, for every point $p$ in the critical area of $\K$, there exists a $p' \in \N(\K,\delta)$ such that $\norm{p-p'}_2^2 \leq \delta$. Formally:

\begin{restatable}{lemma}{epsnet}\label{lem:epsnet}
Given a KDE $\K = (\Kf,M)$ in $\R^d$ with $|M| = n$, and parameter $\delta > 0$, let $\xi = \xi_{\Kf}(1/n)$ and let $\N(\K,\delta)$ be a set that contains all points of the form
\begin{equation*}
    m+\sum_{i=1}^d \frac{k_i\sqrt{\delta}}{\sqrt{d}}e_i,
\end{equation*}
where $m \in M$, $k_i \in \ZZ$, and $-\frac{\sqrt{d}\xi}{\sqrt{\delta}} \leq k_i \leq \frac{\sqrt{d}\xi}{\sqrt{\delta}}$.
Above $e_i$ are the canonical base vectors of $\R^d$. Then for any point $p$ in a $\xi$-ball surrounding one of the points in $M$, there exists a point $p' \in \N(\K,\delta)$ such that $\norm{p-p'}_2^2 \leq \delta$. Moreover, we have that $|\N(\K,\delta)| = n(2\sqrt{d}\xi/\sqrt{\delta})^d$.
\end{restatable}

By checking every point in $\N(\K,\delta)$ and returning one that maximizes $\Kf$, we obtain the following results on finding an approximate mode, which is proven in \Cref{appendix}:
\begin{restatable}{theorem}{lowdimmod}\label{lem:lowdimmod}
    Given a KDE $\K = (\Kf,M)$ in $\R^d$ with $|M| = n$ and a precision parameter $\epsilon > 0$, let $\xi = \xi_{\Kf}(1/n)$ and let $\delta$ be at most the largest number such that $\Kf(c)-\Kf(c+\delta)\leq \epsilon\Kf(c)$ for all $c \leq \xi$. Then we can find an $\epsilon$-approximate mode in $O\left(n(2\sqrt{d}\xi/\sqrt{\delta})^d\right)$. In particular, if $d \leq O(\log^c(n))$, $\xi \leq O(n^c)$, and $\delta \geq O(n^{-c})$ for some constant $c$, then we can find an $\epsilon$-approximate mode in quasi-polynomial time in the number of data points $n$.
\end{restatable}

Our headline result, \Cref{lemma:quais}, follows by combining the dimensional reduction guarantee of \Cref{lem:rds_main} with the observation that for $\bar{\xi} = \max(1, \xi_{\Kf}(1/n))$, choosing 
$\delta = \min\left(\left(\frac{d_2}{c_2}\epsilon\right)^{1/d_2},\frac{\epsilon}{c_2}(2\bar{\xi})^{1-d_2}\right)$
satisfies the requirement of \Cref{lem:lowdimmod} for any relative-distance smooth kernel $\Kf$ with parameters $c_1,d_1,q_1,c_2,d_2$. Moreover, as established in \Cref{lem:xismoothexp}, $\xi_{\Kf}(\frac{1}{n}) \leq c \log^{1/d_1} n$, so we have that the runtime in \Cref{lem:lowdimmod} is $O(n (\log^c n)^d)$ for constant $\epsilon$. 
The claim in \Cref{lemma:quais} for the Cauchy kernel follows by noting that $\xi = n$ and we can take $\delta = \frac{1}{\epsilon}$ in \Cref{lem:lowdimmod}. Finally, note that \Cref{lemma:quais} also includes the polynomial time cost of multiplying the original data set by a random JL matrix.

Overall, we conclude that one can compute an approximate mode in quasi-polynomial time for the Cauchy kernel, or any KDE on a relative-distance smooth kernel, and in particular the approximate mode finding problem for KDEs on Gaussian, logistic, sigmoid, or generalized Gaussian kernels can be solved in quasi-polynomial time. 

\section{Hardness Results}
\label{sec:hardness}

The results from the previous sections place the approximate mode finding problem in quasi-polynomial time for a large class of kernels.
The question arises whether we can do much better;
in this section, we provide some preliminary negative evidence for this possibility. Specifically, 
we prove NP-hardness of finding an approximate mode of a box kernel KDE, where we recall that this kernel takes the form $\Kf(t) = 1$ for $|t| \leq 1$ and $\Kf(t) = 0$ otherwise. Our hardness result is based on the hardness of the $k$-clique problem:
\begin{problem}[$k$-clique]\label{prob:kclique}
Given a $\Delta$-regular graph $G$ and an integer $k$, does $G$ have a complete $k$-vertex subgraph?
\end{problem}
\Cref{prob:kclique} is known to be NP-hard when $k$ is a parameter of the input.
We show how to reduce this problem to KDE mode finding using a reduction inspired by work of Shenmaier on the $k$-enclosing ball problem \cite{Shenmaier:2015}. 
We start by creating a point set given an input $G$ to \Cref{prob:kclique}. Specifically, we embed $G$ in $\R^{|E|}$ as follows: let $P$ to be the set of rows of the incidence matrix of $G$, i.e. the matrix $B$ such that $B_{v,e} = 1$ if $e$ is an edge adjacent to the node $v$ and $B_{v,e} = 0$ otherwise \cite{Shenmaier:2015}. See \Cref{fig:regular} for an example.

\begin{figure}[h]
    \centering
        \vspace{-.5em}
    \includegraphics[width=0.45\textwidth]{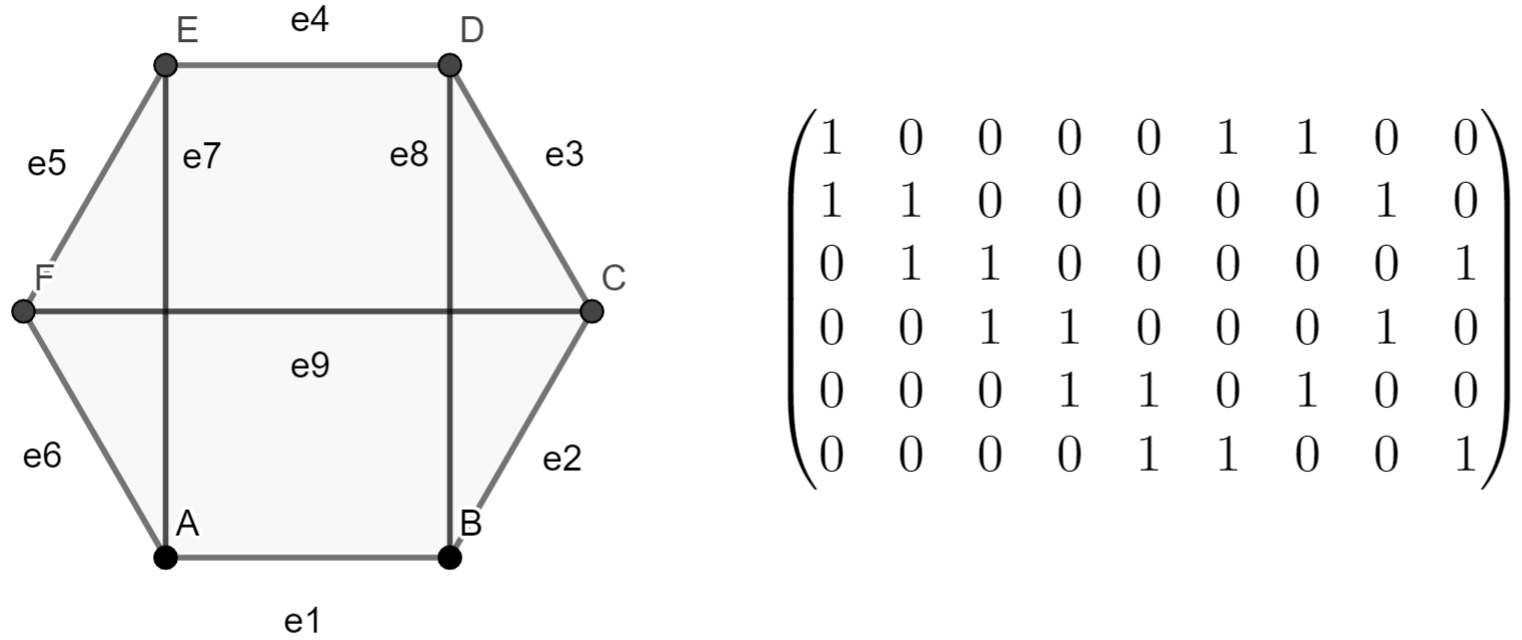}
    \vspace{-1em}
    \caption{An simple $3$-regular graph and its incident matrix $B$.}
    \label{fig:regular}
\end{figure}

We will base our hardness result on the following lemma:
\begin{lemma}[\citet{Shenmaier:2015}]\label{shenmaierlemma}
Given a $\Delta$-regular graph $G = (V,E)$ and integer $k$, let $P$ be defined as above. Let $A = (1-1/k)(\Delta-1)$ and let $R$ be the radius of the smallest ball containing $\geq k$ points in $P$. Then $R^2 \leq A$ if there is a $k$-clique in $G$, and $R^2 \geq A+2/k^2$ otherwise.
\end{lemma}

By rescaling $P$ we can show NP-hardness of the KDE mode finding problem for box kernels:
\begin{restatable}{theorem}{boxhard}\label{thm:boxhard}
The problem of computing a $\frac{1}{n}$-approximate mode of a box kernel KDE is NP-hard.
\end{restatable}
\begin{proof}
The proof follows almost directly from \Cref{shenmaierlemma}. 
Note that the value of the mode of a box kernel KDE is given by the largest number of centers in a ball of radius 1. 
Let $G$ be an instance of \Cref{prob:kclique}, and let $P$ be the set of rows of the incidence matrix of $G$ as described above. 
Now let $M = \{p/\sqrt{A} \; \mid \; p \in P\}$. From the lemma, we know that there is a ball of radius $1$ containing $k$ points if $G$ has a $k$-clique, so $\max_x\bar{\K}_M(x) \geq k$. On the other hand, if $G$ does not have a $k$-clique then every ball of radius $1$ contains at most $k-1$ points, i.e., $\max_x \bar{\K}_M \leq k - 1$. So, an approximation algorithm with error $\epsilon = 1/k \geq 1/n$ can distinguish between the two cases. 
Hence, the ($\epsilon$-approximate) mode finding problem for box kernel KDEs is at least as hard as \Cref{prob:kclique} when $\epsilon \leq 1/n$.
\end{proof}

While it provides an initial result and hints at why the mode finding problem might be challenging, the above hardness result leaves a number of open questions. First off, it does not rule out a constant factor approximation, or a method whose dependence on the approximation parameter $\epsilon$ is exponential (as in our quasi-polynomial time methods). Moreover, the result does not apply for kernels like the Gaussian kernel -- it strongly requires that the value of the box kernel differs significantly between $t=1$ and $t = 1 + \frac{1}{k^2}$. Proving stronger hardness of approximation for the box kernel, or any hardness for kernels used in practice (like a relative-distance smooth kernel) are promising future directions.

\section{Experiments}\label{sec:experiments}
We support our theoretical results with experiments on two datasets, MNIST (60000 data points, 784 dimensions) and Text8 (71290 data points, 300 dimensions). We use both the Gaussian and Generalized Gaussian kernels with a variety of different bandwidths, $\sigma$. A bandwidth of $\sigma$ means that the kernel function as definied in \Cref{sec:prelims} was applied to values $t = \frac{\|m - x\|_2^2}{\sigma^2}$. In general, a large $\sigma$ leads to larger mode value. It also leads to a smoother KDE, which is intuitively easier to maximize. We chose values of $\sigma$ that lead to substantially varying mode values to check the performance of our method across a variety of optimization surfaces.

Since these are high-dimensional datasets, it is not computationally feasible to find an exact mode to compare against. Instead, we obtain a baseline mode value by running the mean-shift heuristic (gradient descent) for 100 iterations, with 60 randomly chosen starting points. To avoid convergence to local optima at KDE centers, these starting points were chosen to be random linear combinations of either all dataset points, or a random pair of points in the data set. The best mode value found was taken as a baseline. 

Once establishing a baseline, we applied JL dimensionality reduction to each data set and kernel for a variety of sketching dimensions, $w$. Again, for efficiency, we use mean-shift to find an approximate low-dimensional mode, instead of the brute force search method from \Cref{sec:low-dim-prob}. We ran for 10 iterations with 30 random restarts, chosen as described above. To recover a high-dimensional mode from our approximate low-dimensional mode, we use \Cref{alg:convex}, since the kernels tested are convex. For each dimension $w$, we ran $10$ trials and report the mean and standard deviation of the KDE value of our approximate mode. Results are included in Figures 2-5. Note that, for visual clarity, the y-axis in these figures does not start at zero.

\begin{figure}[t]
    \centering
    \includegraphics[width=0.48\textwidth]{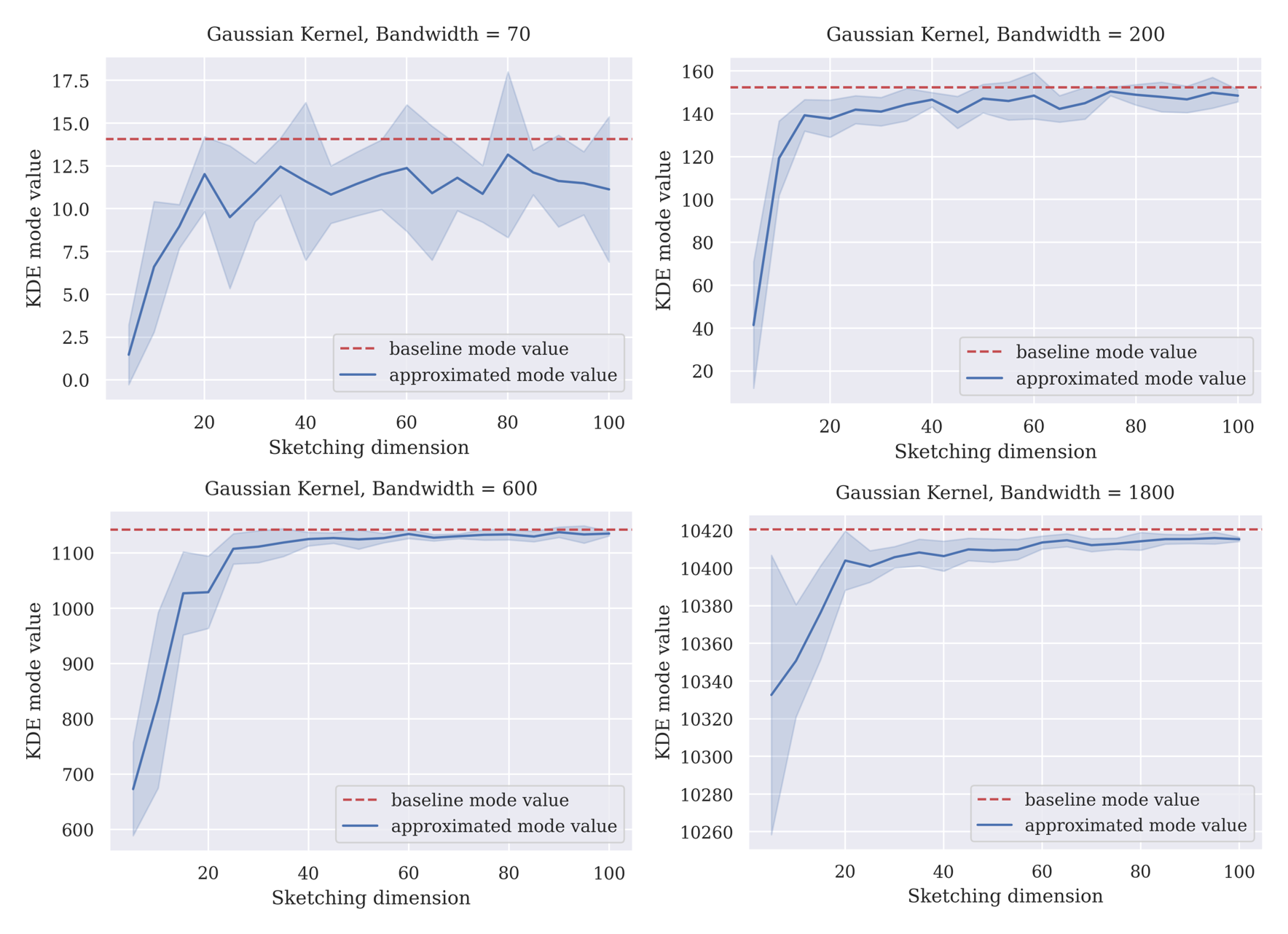}
                        \vspace{-2em}
                        
    \caption{MNIST data using a Gaussian kernel with bandwidths 70, 200, 600, and 1800.}
    \label{fig:mnist_gaussian}
\end{figure}

\begin{figure}[t]
    \centering
    \includegraphics[width=0.48\textwidth]{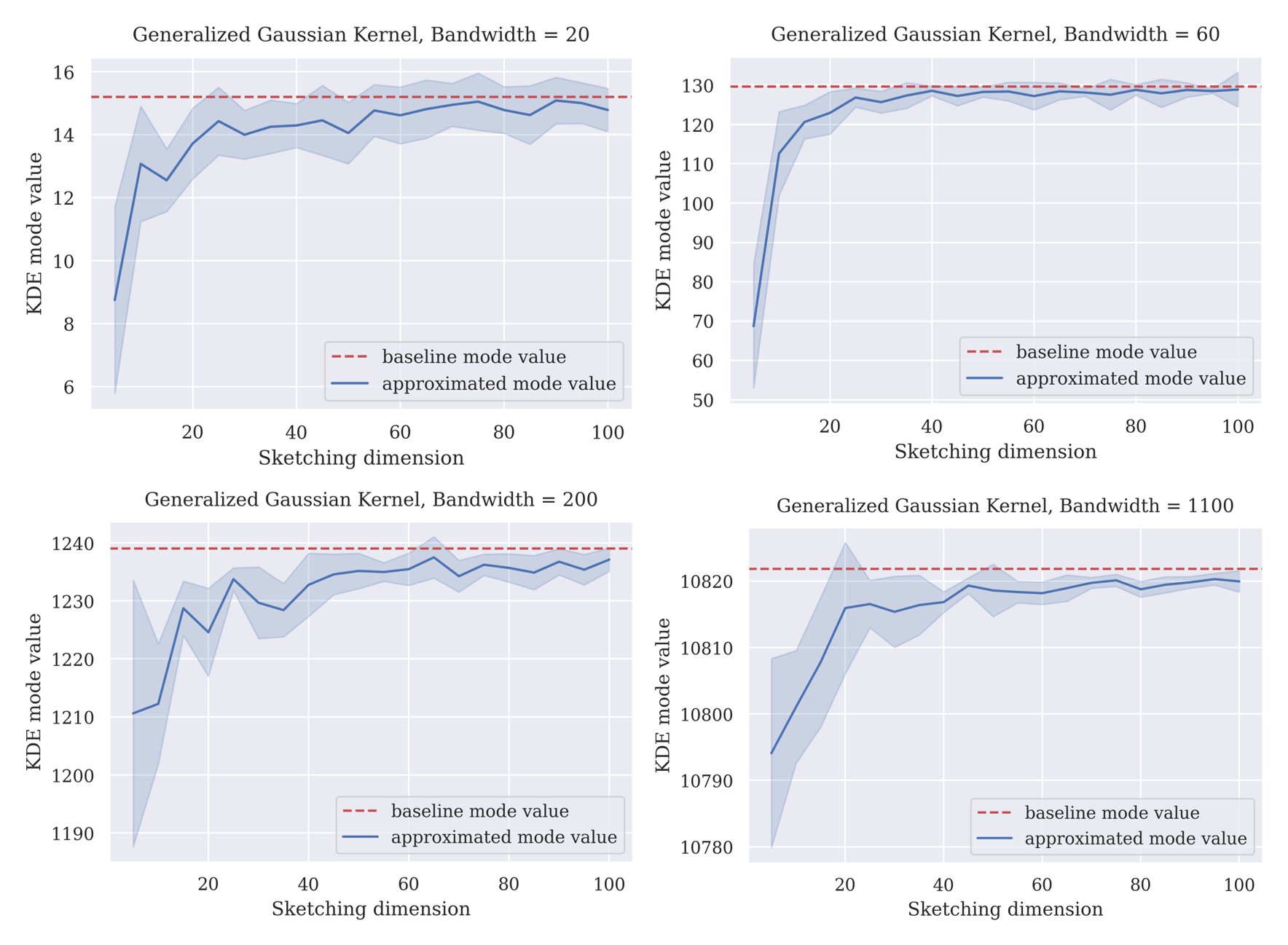}
\vspace{-2em}

        \caption{MNIST data using a Generalized Gaussian kernel with parameter $\alpha = .5$ and bandwidths 20, 60, 200, and 1100.}
    \label{fig:mnist_sigma_gaussian}
\end{figure}

\begin{figure}[h]
    \centering
    \includegraphics[width=0.48\textwidth]{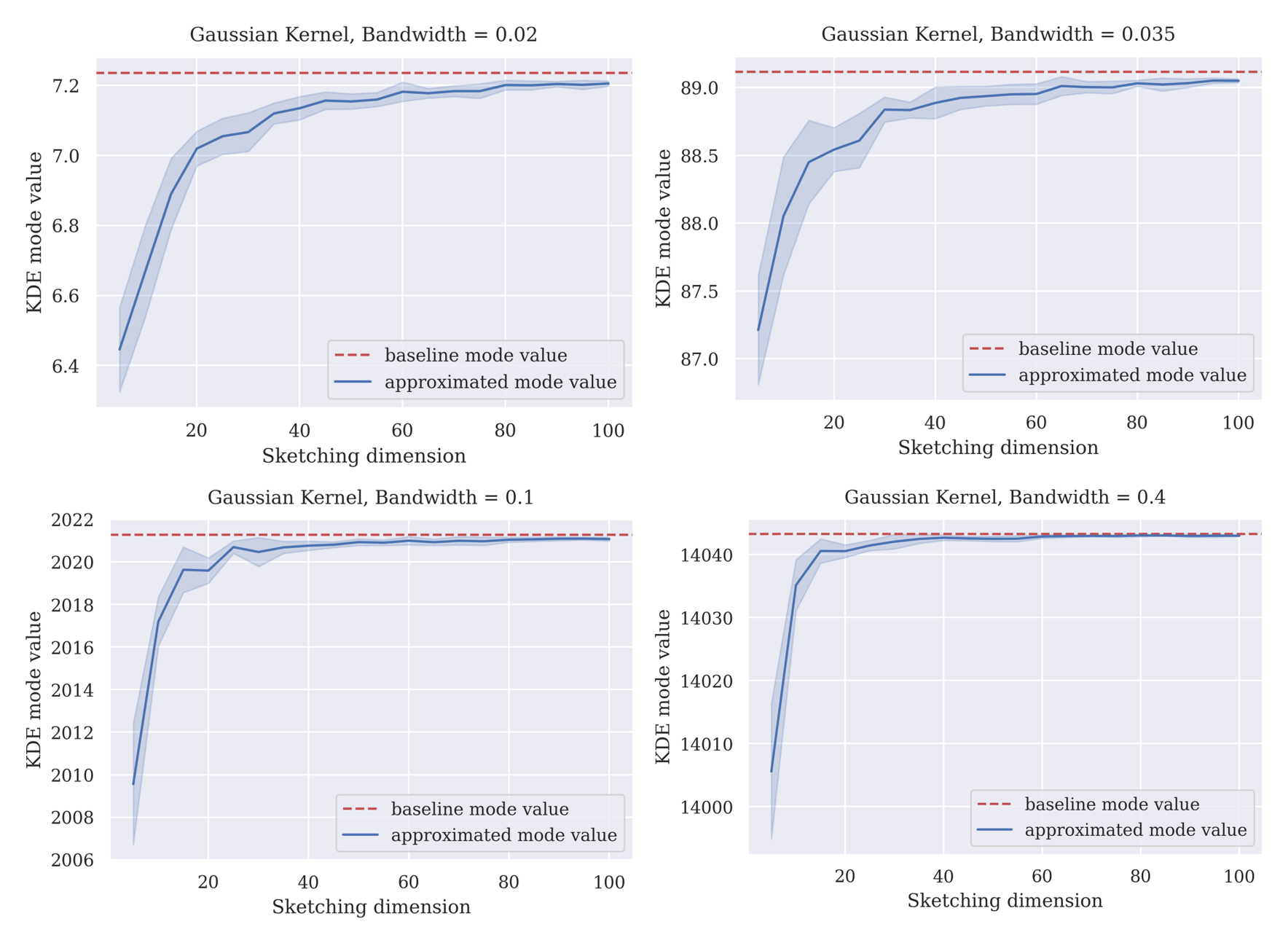}
\vspace{-2em}

            \caption{Text8 data using a Gaussian kernel with parameter with bandwidths .02, .035, .1, and .4.}
    \label{fig:text8_gaussian}
\end{figure}

\begin{figure}[h]
    \centering
    \includegraphics[width=0.48\textwidth]{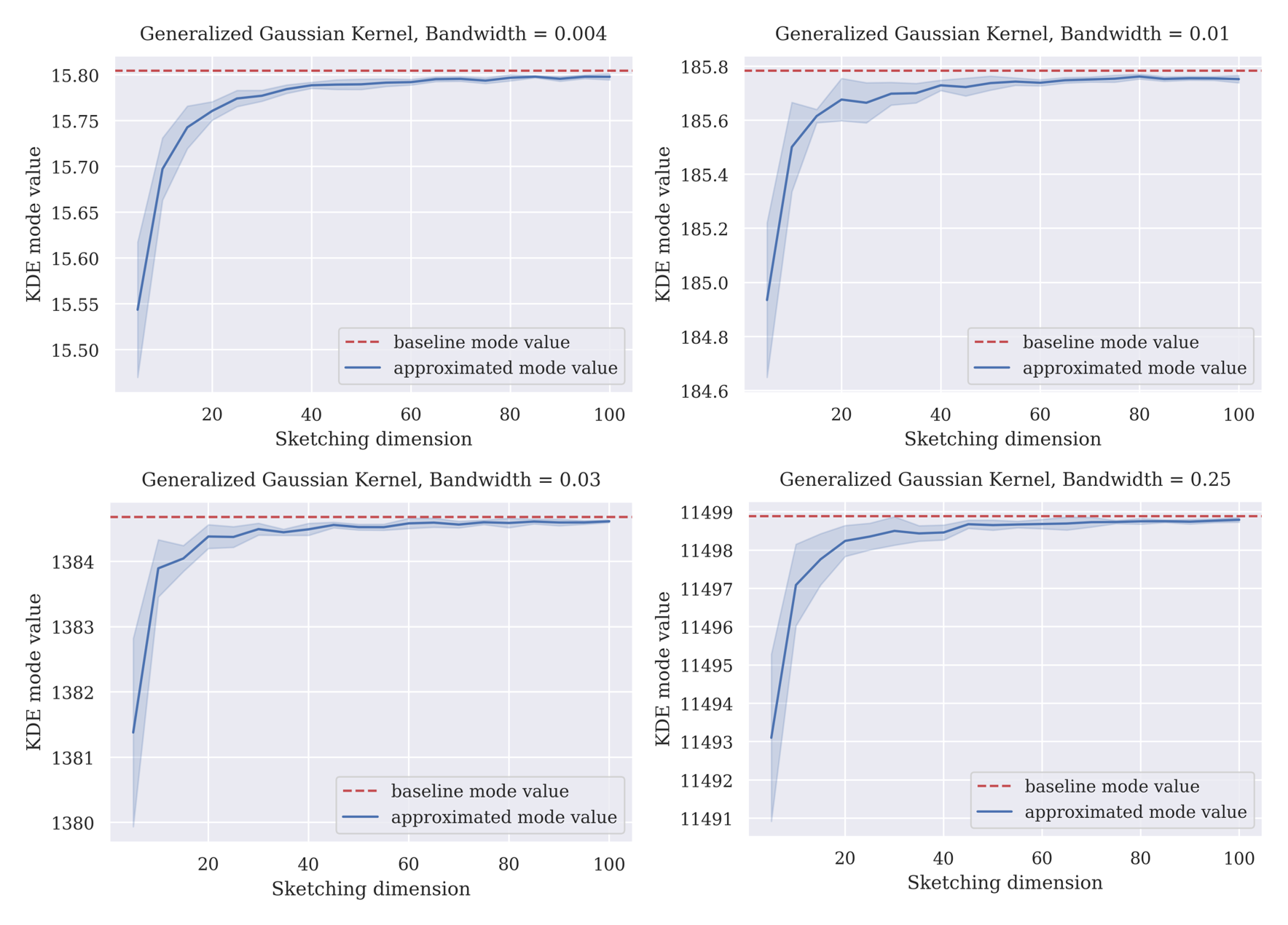}
     \vspace{-2.5em}

    \caption{Text8 data using a Generalized Gaussian kernel with parameter $\alpha = .5$ and bandwidths .004, .01, .03, and .25.}
    \label{fig:text8_sigma_gaussian}
\end{figure}

As apparent from the plots, our Johnson-Lindenstrauss dimensionality reduction approach combined with the mean-shift heuristic performs very well overall. In all cases, it was able to recover an approximate mode with value close to the baseline with sketching dimension $w\ll d$. As expected, performance improves with increasing sketching dimension.

% \section{Discussion}\label{sec:discussion}
% In this article we considered a framework for solving the KDE approximate mode finding problem. We proved that this problem can be solved in quasi-polynomial time for a large class of kernels, and that it is NP-hard in general. Through experiments we have also shown that the dimensionality redction approach works well in practice.

% We have shown complexity results (both positive and negative) for some of the most commonly used kernels. The approximate mode finding problem can be solved in quasi-polynomial time for Gaussian, logistic, sigmoid, and generalized Gaussian kernels, while it is NP-hard for box kernels. However, the complexity for many other kernels is still unknown, most notably the rational quadratic kernel.

% Moreover, it is not known whether quasi-polynomial time is the best we can do. In our proof for the box kernel we noted that this problem is NP-hard for any sufficiently steep kernel. It is still unknown whether this is also true for relative distance smooth kernels.

% \bibliographystyle{apalike}
% \bibliography{Bibliography}

\section{Acknowledgements}
This work was partially funded through NSF Award No. 2045590. Cas Widdershoven's work has been partially funded through the CAS Project for Young Scientists in Basic Research, Grant No. YSBR-040.

\appendix

\nocite{langley00}

\bibliography{arxiv_full_version.bib}
\bibliographystyle{icml2023}

%%%%%%%%%%%%%%%%%%%%%%%%%%%%%%%%%%%%%%%%%%%%%%%%%%%%%%%%%%%%%%%%%%%%%%%%%%%%%%%
%%%%%%%%%%%%%%%%%%%%%%%%%%%%%%%%%%%%%%%%%%%%%%%%%%%%%%%%%%%%%%%%%%%%%%%%%%%%%%%
% APPENDIX
%%%%%%%%%%%%%%%%%%%%%%%%%%%%%%%%%%%%%%%%%%%%%%%%%%%%%%%%%%%%%%%%%%%%%%%%%%%%%%%
%%%%%%%%%%%%%%%%%%%%%%%%%%%%%%%%%%%%%%%%%%%%%%%%%%%%%%%%%%%%%%%%%%%%%%%%%%%%%%%
\newpage
\appendix
\onecolumn
\section{Additional Proofs}\label{appendix}

\subsection{Proofs for \Cref{sec:dimred}}
\mapdownapproxmode*
\begin{proof}
First recall the definitions of $\bar{\K}_{M}(x)$ and $\bar{\K}_{\Pi M}(x)$, which are just fixed positive scalings of ${\K}_{M}(x)$ and ${\K}_{\Pi M}(x)$. It suffices to prove that:
\begin{equation}\label{full_eq_kde}
        (1-\epsilon) \max_{x \in \R^d} \bar{\K}_M(x) \le \max_{x \in \R^w} \bar{\K}_{\Pi M} (x) \le \max_{x \in \R^d} \bar{\K}_{M}(x),
\end{equation}
To prove \eqref{full_eq_kde} we will apply the guarantee of Definition \ref{def:jl} to the $n+1$ points in $\{x^*\}\cup M$, where $x^* \in \argmax_{x \in \R^d} \bar{\K}_{M}(x)$. This guarantee ensures that with probability $(1-\delta)$, $\|a - b\|_2^2 \leq \|\Pi a - \Pi b\|_2^2 \leq (1+\gamma)\|a - b\|_2^2$ for all $a,b$ in this set, where $\Pi \in \R^{w\times d}$ is the JL matrix from the theorem.

We first prove the right hand side of \eqref{full_eq_kde} using an argument identical to the proof of Lemma
10 from \cite{LeeLiMusco:2021}. 
Consider the set of $n$ points $\Pi M$ that contains $\Pi m$ for all $m \in M$. Let $g$ be a map from each point in this set to the corresponding point in $M$. Since $\|\Pi m_1 - \Pi m_2\|_2^2 \geq \|m_1 - m_2\|_2^2$ for all $m_1,m_2\in M$ as
guaranteed above, we have that $g$ is 1-Lipschitz. From Kirszbraun's theorem (\hspace{-.4em}~\Cref{theo:Kirs}) it follows that there is a function $\tilde{g}: \R^d \rightarrow \R^w$  which agrees with $g$ on inputs in $\Pi M$ and satisfies $\|\tilde{g}(s) - \tilde{g}(t)\|_2^2 \leq  \|s-t\|_2^2$ for all $s,t \in \R^d$. 
So for any $x\in \R^w$, there is some $x' = \tilde{g}(x)$ such that, for all $m\in M$, 
\begin{align*}
    \|x' - m\|_2^2 \leq \|x - \Pi m\|_2^2.
\end{align*}
The right hand side of \eqref{full_eq_kde}  then follows: there must be some point $x'$ such that for all $m$, $\|x' - m\|_2^2 \leq \|\tilde{x}^* - \Pi m\|_2^2$ where $\tilde{x}^* \in \argmax_{x \in \R^w} \bar{\K}_{\Pi M}(x)$. Overall we have:
\begin{align*}
    \max_{x \in \R^w} \bar{\K}_{\Pi M} (x) = \bar{\K}_{\Pi M} (\tilde{x}^*) = \sum_{m\in M} \kappa(\|\tilde{x}^* - \Pi m\|_2^2) \leq \sum_{m\in M} \kappa(\|x' - m\|_2^2) \leq \max_{x \in \R^d} \bar{\K}_{M}(x).
\end{align*}
In the second to last inequality we used that $\kappa$ is non-increasing. 

We next prove the left hand side of \eqref{full_eq_kde}. We first have:
\begin{align*}
    \max_{x \in \R^w} \bar{\K}_{\Pi M}(x) &= \max_{x \in \R^w} \sum_{m \in M} \Kf(\norm{x-\Pi m}^2_2) \geq \sum_{m \in M} \Kf(\norm{\Pi x^*-\Pi m}^2_2)  \geq \sumppcondition{\xi} \Kf(\norm{\Pi x^*-\Pi m}^2_2),
\end{align*}
where $x^* \in \argmax_{x \in \R^d} \bar{\K}_M(x)$. Applying the JL guarantee to the $n+1$ points in $\{x^*\}\cup M$, we have that for all $m$, $\norm{\Pi x^* -\Pi m}^2_2\leq (1+\gamma) \norm{x^*-m}^2_2$. So plugging into the equation above, we have: 
\begin{align*}
%\label{eq:no_pi_left}
    \max_{x \in \R^w} \bar{\K}_{\Pi M}(x)& \geq \sumppcondition{\xi} \Kf((1+\gamma)\norm{x^*- m}^2_2)
\end{align*} 
We can then bound:
\begin{align*}
    &\sumppcondition{\xi} \Kf((1+\gamma)\norm{x^*- m}^2_2) \nonumber\\
    & \geq \sumppcondition{\xi} \Kf(\norm{x^*- m}^2_2)+\gamma\norm{x^*-m}_2^2\min_{z\in [\norm{x^* - m}_2^2, (1 + \gamma)\norm{x^* - m}_2^2]} \Kf'(z)\nonumber\\
    & \geq \sumppcondition{\xi} \Kf(\norm{x^*- m}^2_2)+\gamma\cdot\min_{z\in [\norm{x^* - m}_2^2, (1 + \gamma)\norm{x^* - m}_2^2]} \Kf'(z)\cdot z \nonumber\\
    % & \geq \sumppcondition{\xi} \Kf(\norm{x^*- m}^2_2)+\gamma\cdot\min_{z\in [\norm{x^* - m}_2^2, (1 + \gamma)\norm{x^* - m}_2^2]} \frac{\Kf'(z)\cdot z}{\Kf(\norm{x^* - m}_2^2)} \cdot \Kf(\norm{x^* - m}_2^2)\\
    & \geq \sumppcondition{\xi} \Kf(\norm{x^*- m}^2_2)+\gamma\cdot\min_{z\in [\norm{x^* - m}_2^2, (1 + \gamma)\norm{x^* - m}_2^2]} \Kf'(z)\cdot z\cdot\frac{\Kf(\norm{x^* - m}_2^2)}{\Kf(z)}. %\label{eq:second_plug_in} 
\end{align*}
The last inequality follows from the fact that $\kappa$ is non-increasing, so $k'(z)\cdot z$ is negative or zero and $\frac{\Kf(\norm{x^* - m}_2^2)}{\Kf(z)} \geq 1$. Invoking our definition of $\Kf'_{\min}$ and choice of $\gamma = -\frac{\epsilon}{2\Kf'_{\min}}$ we can continue:
\begin{align*}
    &\sumppcondition{\xi} \Kf(\norm{x^*- m}^2_2)+\gamma\cdot\min_{z\in [\norm{x^* - m}_2^2, (1 + \gamma)\norm{x^* - m}_2^2]} \Kf'(z)\cdot z\cdot\frac{\Kf(\norm{x^* - m}_2^2)}{\Kf(z)} \\
    & \geq \sumppcondition{\xi} \Kf(\norm{x^*- m}^2_2)+\gamma\cdot \Kf'_{\min}\cdot \Kf(\norm{x^* - m}_2^2)\\
    & = \sumppcondition{\xi} \Kf(\norm{x^*- m}^2_2) -\epstwo\cdot \Kf(\norm{x^* - m}_2^2)\\
    & = \left(1-\epstwo\right)\sumppcondition{\xi} \Kf(\norm{x^*- m}^2_2)\\
    & = \left(1-\epstwo\right)\left(\sum_{m \in M} \Kf(\norm{x^*- m}^2_2) -\sumppconditionp{\xi} \Kf(\norm{x^*- m}_2^2)\right)\\
    & \geq \left(1-\epstwo\right)\left(\sum_{m \in M} \Kf(\norm{x^*- m}^2_2) - \epstwo\right) \\
    & = \left(1-\epstwo\right)\left(\max_{x \in \R^d} \bar{\K}_{M}(x)-\epstwo\right) 
    \geq \left(1-\epstwo\right)^2\max_{x \in \R^d} \bar{\K}_{M}(x)
     \geq \left(1-\epsilon\right)\max_{x \in \R^d} \bar{\K}_{M}(x)
\end{align*}
Note that in the second to last line we invoked the definition of $\xi \geq \xi_{\Kf}(\frac{\epsilon}{2n})$. Specifically, we used that, for any $m$ with $\norm{x^*-m}_2^2 > \xi$, $\Kf(\norm{x^*- m}^2_2) \leq \frac{\epsilon}{2n}$. In the last line we use that $\max_{x \in \R^d} \bar{\K}_{M}(x) \geq 1$.
\end{proof}

\begin{restatable}{lemma}{xismoothexp}\label{lem:xismoothexp}
    Let $\K_M = (\Kf, M)$ be a KDE for a point set $M$ with cardinality $n$ and relative-distance smooth kernel $\Kf$ with parameters $c_1, d_1,q_1,c_2,d_2$. Then for any $\epsilon \in (0,1]$,
    $\xi_{\Kf}(\frac{\epsilon}{2n}) \leq c \log^{1/d_1}\left(\frac{2n}{\epsilon}\right)$ for a fixed constant $c$ that depends on $c_1,d_1,q_1,c_2,$ and $d_2$.
\end{restatable}
\begin{proof}
Since $\Kf$ is positive, non-increasing, and $\Kf(0) = 1$ we can write $\Kf(t) = e^{-f(t)}$ for some positive, non-decreasing function $f$ with $f(0) = 0$. We have $\frac{-\Kf'(t)t}{\Kf(t)} = f'(t)t \geq c_1t^{d_1} - q_1$ and thus $f'(t) \geq \max(0,c_1t^{d_1 - 1} - \frac{q_1}{t})$. Writing $\Kf(t) = e^{-\int_0^t f'(x)dx}$, we will upper bound $\Kf(t)$ by lower bounding $\int_0^t f'(x)dx$. Specifically, we have:
\begin{align*}
    \int_0^t f'(x)dx &\geq \int_0^t \max\left(0,c_1x^{d_1 - 1} - \frac{q_1}{x}\right) dx\\
    & = \int_0^t c_1x^{d_1 - 1} dx - \int_0^t \min(c_1x^{d_1 - 1},\frac{q_1}{x}) dx\\
    &= \frac{c_1}{d_1}t^{d_1} - \int_0^{(q_1/c_1)^{1/d_1}}c_1x^{d_1 - 1} dx - \int_{(q_1/c_1)^{1/d_1}}^t \frac{q_1}{x} dx \\
    &= \frac{c_1}{d_1}t^{d_1} - \frac{q_1}{d_1} - q_1\log(t) + \frac{q_1}{d_1}\log(q_1/c_1). 
\end{align*}
It follows that 
$
     \Kf(t) \leq e^{-\frac{c_1}{d_1}t^{d_1} + \frac{q_1}{d_1} + q_1\log(t) - \frac{q_1}{d_1}\log\frac{q_1}{c_1}}.
$
We want to upper bound the smallest $t$ such that $\Kf(t) \leq \frac{\epsilon}{2n}$. 
Let $z$ be a sufficiently large constant so that:
\begin{align*}
    \frac{c_1}{d_1}z^{d_1} \geq 2\left(\frac{q_1}{d_1} + q_1\log(z) - \frac{q_1}{d_1}\log\frac{q_1}{c_1}\right)
\end{align*}
Then it suffices to pick $t \geq \max\left(z, \left(\frac{d_1}{c_{1}}\log(\frac{2n}{\epsilon})\right)^{1/d_1}\right) = O\left(\log^{1/d_1}(\frac{2n}{\epsilon})\right)$ to ensure that $\Kf(t) \leq \frac{\epsilon}{2n}$.
\end{proof}

\rdsmain*
\begin{proof}
With \Cref{lem:xismoothexp} in place, our main result for relatively distance smooth kernels follows directly:
By~\Cref{lem:xismoothexp}, $\xi =c\log^{1/d_1}\left(\frac{2n}{\epsilon}\right) \geq \xi_{\Kf}(\frac{\epsilon}{2n})$. And since $\Kf$ is relative-distance smooth, we have that:
\begin{align*}
    \min_{0 \leq x \leq 2\xi} \frac{\Kf'(x)x}{\Kf(x)} & \geq \min_{0 \leq x \leq 2\xi} -c_{2}x^{d_2} = - c_{2}(2\xi)^{d_2}  \geq -c'\log^{d_2/d_1}\left(\frac{2n}{\epsilon}\right),
\end{align*}
for sufficiently large constant $c'$. Let $\Kf'_{\min} = -c'\log^{d_2/d_1}\left(\frac{2n}{\epsilon}\right)$. 
Invoking \Cref{thm:mapdownapproxmode}, we require that $\gamma = -\frac{\epsilon}{2\Kf'_{\min}} = \frac{\epsilon}{2c'}\log^{-d_2/d_1}\left(\frac{2n}{\epsilon}\right)$.
\end{proof}

\subsection{Proofs for \Cref{sec:modrec}}
\kdemeanshift*
\begin{proof}
	We will show that:
	\begin{align}
		\label{eq:main_meanshift}
		{\K}_{M} (x') \geq {\K}_{\Pi M} (\tilde{x})
	\end{align}
where $\tilde{x} \in \R^w$ is the approximate maximizer of ${\K}_{\Pi M}$, as defined in the theorem statement. 
If we can prove \eqref{eq:main_meanshift} then the theorem follows by the following chain of inequalities:
	\begin{align*}
	{\K}_{M} (x') \geq {\K}_{\Pi M} (\tilde{x}) \geq (1-\alpha)\max_{x}{\K}_{\Pi M} ({x})\geq (1-\alpha)(1-\epsilon)\max_{x}{\K}_{M} ({x}) \geq  (1-\alpha-\epsilon)\max_{x}{\K}_{M} ({x}).
	\end{align*}
To prove \eqref{eq:main_meanshift}, we follow a similar approach to \cite{LeeLiMusco:2021}. Since $\Pi$ satisfies \eqref{eq:lip_pi}, the function $f$ mapping $\{\Pi x^*\}\cup \Pi M \rightarrow \{x^*\}\cup M$ is 1-Lipschitz. Accordingly, by Kirszbraun's Extension Theorem (\Cref{theo:Kirs}), there is some function $g(x): \R^w \rightarrow \R^d$ that agrees with $f$ on inputs from $\{\Pi x^*\}\cup \Pi M$ and remains 1-Lipschitz. It follows that, if we apply $g$ to $\tilde{x}$ then for all $m\in M$,
\begin{align*}
	\norm{{g}(\tilde{x}) - m}_2 \le \norm{\tilde{x} - \Pi m}_2. 
\end{align*}
In other words, for our approximate low-dimensional mode $\tilde{x}$, there is a high-dimensional point $g(\tilde{x})$ that is closer to all points in $M$ than $\tilde{x}$ is to the points in $\Pi M$. In fact, by extending all points by one extra dimension, we can obtain an exact equality. In particular, let $x'' \in \R^{d+1}$ equal $g(\tilde{x})$ on its first $d$ coordinates, and $0$ on its last coordinate. For each $m \in M$ let $m'' \in \R^{d+1}$ equal $m$ on its first $d$ coordinates, and $\sqrt{\norm{\tilde{x} - \Pi m}_2^2 - \norm{{g}(\tilde{x} ) - m}_2^2}$ on its last coordinate. Let $M'' \subset \R^{d\times 1}$ denote the set of these transformed points. First observe that for all $m''\in M''$
\begin{align}
	\label{eq:eqdistances}
	\norm{x'' - m''}_2 = \norm{\tilde{x} - \Pi m}_2. 
\end{align}
Accordingly, $x'$ as defined in the theorem statement is exactly equivalent to the first $d$ entries of the $d+1$ dimensional vector that would be obtained from running one iteration of mean-shift on $x''$ using point set $M''$. Call this $d+1$ dimensional vector $\bar{x}'$. By \Cref{theo:Meer}, we have that:
\begin{align*}
	\K_{M''}(\bar{x}') \geq  \K_{M''}(x'') =  \K_{\Pi M}(\tilde{x}). 
\end{align*}
The last equality follows from \eqref{eq:eqdistances}. Finally, for any non-increasing kernel we have that:
\begin{align*}
	\K_{M}({x}') \geq  \K_{M''}(\bar{x}'),
\end{align*}
because $\|x' - m\|_2^2 \leq \|\bar{x}' - m''\|_2^2$ for all $m$. This is simply because $x'$ and $m$ are equal to $\bar{x}'$ and $m''$, but with their last entry removed, so they can only be closer together. Combining the previous two inequalities proves \eqref{eq:main_meanshift}, which establishes the theorem. 
\end{proof}

\fullconvex*
\begin{proof}
\Cref{thm:fullconvex} immediately follows by combining \Cref{thm:mapdownapproxmode} with \Cref{thm:kde_meanshift}. In particular, if $\Pi$ is chosen with $w = O\left(\frac{\log((n+1)/\delta)}{\min(1,\gamma^2)}\right)$ rows (as in \Cref{alg:convex}) then with probability $1-\delta$, we have that both \eqref{eq:theorem2_equation} and \eqref{eq:lip_pi} hold with probability $1-\delta$, which are the only conditions needed for \Cref{thm:kde_meanshift} to hold.
\end{proof}

\subsection{Proofs for \Cref{sec:low-dim-prob}}
\epsnet*
\begin{proof}
The second claim on the size of $\N(\K,\delta)$ is immediate. For the first claim, note that $p$ can be written as $p''+\sum_i \frac{k'_i\sqrt{\delta}}{\sqrt{d}}e_i$ with $p'' \in M$ and $|k'_i| \leq \frac{\sqrt{d}\xi}{\sqrt{\delta}}$. Let $p' = p''+\sum_i \frac{\lfloor k'_i\rfloor\sqrt{\delta}}{\sqrt{d}}e_i \in \N(\K,\delta)$. Then we have
\begin{align*}
    \norm{p-p'}_2^2 
    & = \norm{p''+\sum_i \frac{k'_i\sqrt{\delta}}{\sqrt{d}}e_i - p''-\sum_i \frac{\lfloor k'_i\rfloor\sqrt{\delta}}{\sqrt{d}}e_i}_2^2 = \norm{\sum_{i=1}^d \frac{(k'_i - \lfloor k'_i\rfloor)\sqrt{\delta}}{\sqrt{d}}e_i}^2 \\
    & = \sum_{i=1}^{d} \left(\frac{(k'_i - \lfloor k'_i\rfloor)\sqrt{\delta}}{\sqrt{d}}\right)^2 \leq \sum_{i=1}^{d} \left(\frac{\sqrt{\delta}}{\sqrt{d}}\right)^2 = \delta.
\end{align*}
\end{proof}

\lowdimmod*
\begin{proof}
Since there always exists a mode in the critical area of $\K$, we can use~\Cref{lem:epsnet} to find a point $p'$ at most $\delta$ away from a mode $p$ of $\K$ in $O\left(n(2\sqrt{d}\xi/\sqrt{\delta})^d\right)$. Then we have
\begin{align*}
\K(p') & = \sum_{m\in M}\Kf(\|m-p'\|_2^2) \geq \sum_{m\in M}\Kf(\|m-p\|_2^2 + \|p-p'\|_2^2)  \geq \sum_{m\in M}\Kf(\|m-p\|_2^2 + \delta) \\
& \geq \sum_{m\in M}\Kf(\|m-p\|_2^2)) - \epsilon \Kf(\|m-p\|_2^2)) = (1-\epsilon) \sum_{m\in M}\Kf(\|m-p\|_2^2)) = (1-\epsilon)\K(p)
\end{align*}
\end{proof}

\subsection{Analysis for Relative-Distance Smooth Kernels}
Let $\bar{\xi} = \max(1, \xi_{\Kf}(1/n))$. We will prove that for any relative distance smooth kernel $\Kf$ with parameters $c_1, d_1, q_1, c_2$, and $d_2$, we have $\Kf(c)-\Kf(c+\delta)\leq \epsilon\Kf(c)$ for all $c \leq \xi = \xi_{\Kf}(1/n)$ as long as: 
\begin{align*}
\delta = \min\left(\left(\frac{d_2}{c_2}\epsilon\right)^{1/d_2},\;\frac{\epsilon}{c_2}(2\bar{\xi})^{1-d_2}\right).
\end{align*}
By the definition of relative distance smooth kernels, we have that $-\Kf'(t) \leq c_2 t^{d_2-1}\Kf(t)$. Hence,
\begin{align*}
\Kf(c) - \Kf(c+\delta) \leq \int_c^{c+\delta}c_2 t^{d_2-1}\Kf(t) dt \leq \Kf(c)\int_c^{c+\delta}c_2t^{d_2-1}dt.
\end{align*}
The last step follows from the fact that $\kappa(t)$ is non-increasing in $t$.
Since $d_2 > 0$, this simplifies to
\begin{align*}
\Kf(c)-\Kf(c+\delta) \leq \Kf(c)\int_c^{c+\delta}c_2t^{d_2-1}dt = \frac{c_2}{d_2}\Kf(c)((c+\delta)^{d_2} - c^{d_2}).
\end{align*}
So, we need to show that $\frac{c_2}{d_2}\left((c+\delta)^{d_2} -c^{d_2}\right) \leq \epsilon$. We consider two cases:

\textbf{Case 1:} When $d_2$ is $<1$, consider the function $f(x) = (x+\delta)^{d_2}-x^{d_2}$. This function is non-increasing, indeed, $f'(x) = d_2((x+\delta)^{d_2-1}-x^{d_2-1}) < 0$. Hence, we have that 
\begin{align*}
((c+\delta)^{d_2}-c^{d_2}) \leq \delta^{d_2}.
\end{align*}
We can pick $\delta = (\frac{d_2}{c_2}\epsilon)^{1/d_2}$.

\textbf{Case 2:} On the other hand, when $d_2 \geq 1$, the function $f(x) = x^{d_2}$ is convex, so we have:
\begin{align*}
(c+\delta)^{d_2} -c^{d_2} \leq \delta f'(c+\delta) = \delta d_2 (c+\delta)^{d_2-1} \leq \delta d_2 2^{d_2-1}\max(\xi^{d_2-1}, \delta^{d_2-1})\leq \delta d_2 2^{d_2-1}\bar{\xi}^{d_2 - 1}
\end{align*}
In this case, we can choose $\delta = \frac{\epsilon}{c_2}(2\bar{\xi})^{1-d_2}$. 

Hence, picking $\delta = \min\left(\left(\frac{d_2}{c_2}\epsilon\right)^{1/d_2},\;\frac{\epsilon}{c_2}(2\bar{\xi})^{1-d_2}\right)$ ensures that $(c+\delta)^{d_2} - c^{d_2} \leq \frac{d_2}{c_2}\epsilon$, and thus that $\frac{c_2}{d_2}\left((c+\delta)^{d_2} - c^{d_2} \right) \leq \epsilon$, as required.

% \subsection{Proofs for \Cref{sec:hardness}}
% \boxhard*
% \begin{proof}
% The proof follows almost directly from \Cref{shenmaierlemma}. 
% Note that the value of the mode of a box kernel KDE is given by the largest number of centers in a ball of radius 1. 
% Let $G$ be an instance of \Cref{prob:kclique}, and let $P$ be the set of rows of the incidence matrix of $G$ as described above. 
% Now let $M = \{p/\sqrt{A} \; \mid \; p \in P\}$. From the lemma, we know that there is a ball of radius $1$ containing $k$ points if $G$ has a $k$-clique, so $\max_x\bar{\K}_M(x) \geq k$. On the other hand, if $G$ does not have a $k$-clique then every call of radius $1$ contains at most $k-1$ points, i.e., $\max_x \bar{\K}_M \leq k - 1$. So, an approximation algorithm with error $\epsilon = 1/k \geq 1/n$ can distinguish between the two cases. 
% Hence, the ($\epsilon$-approximate) mode finding problem for box kernel KDEs is at least as hard as \Cref{prob:kclique} when $\epsilon \leq 1/n$.
% \end{proof}

\section{Mode Recovery for Non-convex Kernels}
\label{sec:nonconvexrecovery}
In \Cref{sec:convexrecovery}, we show that the mean-shift method can rapidly recover an approximate mode for any convex, non-increasing kernel from an approximation to the JL reduced problem. In this section, we briefly comment on an alternative method that can also handle non-convex kernels, albeit at the cost of worse runtime. Specifically, it is possible to leverage a recent result from \cite{BiessKontorovichMakarychev:2019} on an algorithmic version of the Kirszbraun extension theory. This work provides an algorithm for explicitly extending a function $f$ that is Lipschitz on some fixed set of points to \emph{one additional} point. The main result follows:

\begin{theorem}[\cite{BiessKontorovichMakarychev:2019}]\label{theo_biess}
    Consider a finite set $(x_i)_{i\in [n]}\subset X = \R^w$, and its image $(y_i)_{i\in [n]}\subset Y = \R^d$ under some $L$-Lipschitz map $f: X \rightarrow Y$. There is an algorithm running in $O(nw + nd\log n/\epsilon^2)$ time which returns, for any point $z\in \R^w$, and a precision parameter $\epsilon > 0$, a point $z'\in \R^d$ satisfying   for all $i \in [n]$,
    \begin{align*}
    \norm{z' - f(x_i)}^2 \le (1 + \epsilon)L\norm{z - x_i}^2
    \end{align*}
\end{theorem}
From this result we can obtain a claim comparable to \Cref{thm:fullconvex}:
\begin{restatable}{theorem}{nonconvex}\label{thm:non_convex}
Let $\K_M = (\Kf, M)$ be a $d$-dimensional shift-invariant KDE where $\Kf$ is differentiable and non-increasing but not necessary convex. Let $\epsilon$ and $\gamma$ (which depends on $\Kf$ and $\epsilon$) be as in \Cref{thm:mapdownapproxmode} and let $\Pi$ be a random JL matrix as in \Cref{def:jl} with $w = O\left(\frac{\log((n+1)/\delta)}{\min(1,\gamma^2)}\right)$ rows. Let $\tilde{x}$ be an approximate maximizer for ${\K}_{\Pi M}$ satisfying ${\K}_{\Pi M} (\tilde{x}) \geq (1-\alpha)\max_{x \in \R^w} {\K}_{\Pi M} (x)$. If we run the algorithm of \Cref{theo_biess} with $X = \Pi M$, $Y=M$, $z=\tilde{x}$, and error parameter $\gamma$, then the method returns $x' \in \R^d$ satisfying:
    \begin{align*}
        {\K}_{M} (x') \geq (1-2\epsilon - \alpha)\max_{x \in \R^w} {\K}_{M} (x).
    \end{align*}
\end{restatable}

\begin{proof}
For conciseness, we sketch the proof. As discussed, by \Cref{def:jl}, $\Pi M \rightarrow M$ is a 1-Lipschitz map. So it follows that the $x'$ returned by the algorithm of \cite{BiessKontorovichMakarychev:2019} satisfies for all $m\in M$,
    \begin{align*}
    \norm{x' - m}^2 \le (1 + \gamma)\norm{\tilde{x} - \Pi m}^2
    \end{align*}
It follows that:
\begin{align*}
\K_{M}(x') \geq \sum_{m\in M} \Kf\left((1+\gamma)\norm{\tilde{x} - \Pi m}^2\right).
\end{align*}
By the same argument used in the proof of \Cref{thm:mapdownapproxmode}, we have that 
\begin{align*}
\sum_{m\in M} \Kf\left((1+\gamma)\norm{\tilde{x} - \Pi m}^2\right) \geq \sum_{m\in M} (1-\epsilon)\Kf\left(\norm{\tilde{x} - \Pi m}^2\right) = (1-\epsilon)\K_{\Pi M}(\tilde{x}).
\end{align*}
In turn, since \Cref{thm:mapdownapproxmode} holds under the same conditions as \Cref{thm:non_convex}, we have:
\begin{align*}
\K_{\Pi M}(\tilde{x}) \geq (1-\alpha) \max_x \K_{\Pi M}(x) \geq (1-\alpha) (1-\epsilon)\max_x \K_{M}(x).
\end{align*}
The result follows from noting that $(1-\alpha) (1-\epsilon)^2 \geq (1-2\epsilon-\alpha$). By rescaling $\epsilon$ we can obtain equivalent precision to \Cref{thm:fullconvex}.
\end{proof}

Note that for the common relative-distance smooth kernels addressed in \Cref{cor:common_recover}, we have that $\gamma = O({\log(n /\epsilon)}/{\epsilon})$. So, the runtime of recovering a high-dimensional model using the method of \cite{BiessKontorovichMakarychev:2019} is $O(nd\log^3(n)/\epsilon^2)$. This exceeds the $O(nd)$ runtime of the mean-shift method. However, in contrast to mean-shift, the method can be applied to non-convex kernels like generalized Gaussian kernels with $\alpha > 1$.

\end{document}